\documentclass[journal]{IEEEtran}

\usepackage{amsmath,amssymb,amsfonts}
\usepackage{amsthm}
\usepackage{algorithmic}
\usepackage{algorithm}
\usepackage{array}
\usepackage[caption=false,font=normalsize,labelfont=sf,textfont=sf]{subfig}
\usepackage{textcomp}
\usepackage{stfloats}
\usepackage{url}
\usepackage{verbatim}
\usepackage{graphicx}
\usepackage{cite}
\usepackage{booktabs}
\usepackage{pifont}
\usepackage{xspace}
\usepackage[hidelinks, bookmarksnumbered=true]{hyperref}

\newtheorem{theorem}{Theorem}

\newtheorem{lemma}[theorem]{Lemma}
\newtheorem{fact}[theorem]{Fact}
\newtheorem{remark}[theorem]{Remark}

\newtheorem{assumption}[theorem]{Assumption}

\newcommand{\etal}{\textit{et al.}\xspace}
\newcommand{\E}[1]{\mathbb{E}\left[#1\right]}
\newcommand{\normsq}[1]{\left\|#1\right\|^2}
\newcommand{\inner}[2]{\left\langle #1, #2 \right\rangle}
\newcommand{\ie}{i.e.,\xspace}
\newcommand{\eg}{e.g.,\xspace}

\begin{document}

\title{Federated Client Selection under Partial Visibility: A POMDP Approach with Spatio-Temporal Attention}

\author{
Qijun Hou,
Yuchen Shi,
Pingyi Fan,~\IEEEmembership{Fellow,~IEEE}
and Khaled B. Letaief,~\IEEEmembership{Fellow,~IEEE}%
\thanks{Q. Hou, Y. Shi, and P. Fan are with the Department of Electronic Engineering, BNRist, Tsinghua University, Beijing, China (e-mail: \{hqj23, shiyc21\}@mails.tsinghua.edu.cn, fpy@mail.tsinghua.edu.cn).}%
\thanks{K. B. Letaief is with the Department of Electronic and Computer Engineering, Hong Kong University of Science and Technology, Hong Kong (e-mail: eekhaled@ust.hk).}%
}

\markboth{IEEE Internet of Things Journal}%
{Hou \textit{et al.}: Federated Client Selection under Partial Visibility: A POMDP Approach}

\maketitle

\begin{abstract}
In Internet-of-Things (IoT) deployments, federated learning (FL) must address both data heterogeneity and a dynamic set of clients visible to the server, since IoT devices are intermittently available due to energy limitations, mobility, and connectivity constraints.
In contrast to most existing methods that assume full visibility, we formulate FL under partial visibility as a Partially Observable Markov Decision Process (POMDP) and propose a reinforcement learning framework that adapts client selection to the visible client set of each round.
Building on a multi-step Deep Q-Learning solution, we design a Spatio-Temporal Attention-based Q-Network that jointly encodes variable-length client observations and historical global models, and we further establish a convergence guarantee for the global model under a stable Q-network.
By integrating historical global information with client identity embeddings, the proposed method captures temporal training dynamics and client-specific characteristics across rounds.
Experiments on CIFAR-10, Fashion-MNIST, and UCI-HAR under two partial-visibility patterns and heterogeneous data distributions show that our approach achieves competitive or the best accuracy in the majority of settings while reducing communication and computation cost, suggesting its practical applicability to IoT and edge deployments.
\end{abstract}


\IEEEPARstart{F}{ederated} Learning (FL) enables collaborative model training across many distributed clients while preserving data privacy. A fundamental challenge in FL arises from data heterogeneity, where client data distributions are not independent and identically distributed (non-IID), which substantially hinders model convergence and generalization. To address this issue, a large body of prior work has investigated advanced aggregation strategies and client selection mechanisms to better cope with non-IID data.

In addition to data heterogeneity, system heterogeneity, such as varying computation capabilities, communication latency, and availability, also affects federated learning systems.
Partial visibility, in which only a subset of clients is visible to the server in each communication round, is a hallmark of IoT and edge deployments rather than a conventional form of system heterogeneity.
The set of clients visible to the aggregation server changes from round to round: devices enter and leave the visible set due to battery-driven duty-cycling, mobility, and intermittent wireless connectivity, and an edge server or a mobile server, such as a vehicular or drone-mounted aggregation node, may further restrict the visible clients to its current coverage region, as studied in \textit{EdgeFLow}~\cite{edgeflow}.
In each communication round, the server therefore observes only a subset of clients, its current cluster, rather than the full population.
We refer to this time-varying visibility as partial visibility, and the central challenge it poses for FL is adaptation: the training must continually adjust to whatever clients are visible each round.
Fig.~\ref{fig:partialvis} illustrates the two typical sources of partial visibility.
Motivated by these IoT scenarios, this work aims to enhance the performance of FL under partial visibility by adapting to the time-varying visible client set.

\begin{figure}[!t]
    \centering
    \subfloat[Mobile Server]{%
        \includegraphics[width=\columnwidth]{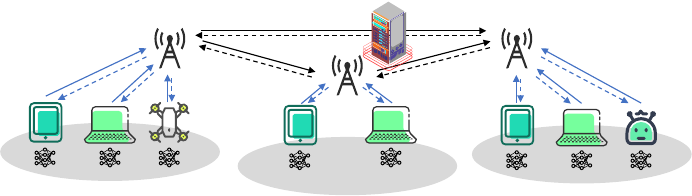}%
        \label{subfig:mobile}%
    }
    \hfil
    \subfloat[Random Availability]{%
        \includegraphics[width=0.80\columnwidth]{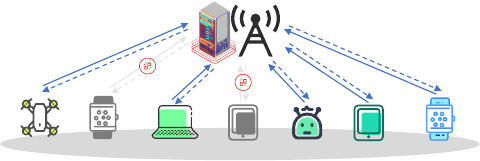}%
        \label{subfig:random}%
    }
    \caption{Illustration of two representative partial visibility scenarios. (a) depicts the scene where the server has to move across different regions; (b) depicts the scene where clients intentionally close the connection.\label{fig:partialvis}}
\end{figure}

A substantial body of research has explored a variety of approaches to tackle heterogeneity, among which client selection and aggregation strategies play a central role.
Representative existing methods select participating clients based on explicit evaluation metrics (\eg validation performance or loss reduction), heuristic rules (\eg resource-aware or fairness-driven criteria), or learning-based policies, including reinforcement learning (RL) formulations.
In particular, RL-based methods are considered more adaptive under severe heterogeneity. Existing studies, \eg FAVOR~\cite{favor}, FLASH-RL~\cite{flashrl}, FedAgent~\cite{fedagent}, commonly formulate client selection in federated learning under data heterogeneity as a Markov Decision Process (MDP), and leverage RL to learn adaptive selection policies.
However, these approaches typically rely on the assumption that the server has access to information about the entire client population in every communication round to guide the selection process, leaving client selection under partial visibility as an under-explored problem that calls for principled modeling and algorithmic solutions. Existing methods that consider partial visibility address it mainly through aggregation-side corrections or selection rules based only on the current round~\cite{f3ast,fedawe,edgeflow}, and often ignore the training history in client selection.

Inspired by these works, we consider client selection as a sequential decision-making problem, while also accounting for partial visibility.
In this context, we naturally formulate client selection as a Partially Observable Markov Decision Process (POMDP), since the server's observations provide incomplete information about the underlying global training state.
To make effective decisions in a POMDP, existing solution methods incorporate the history of past observations and actions.
Such historical context enables the server to reason about the relationship between the underlying states and the observations, mitigating the uncertainty introduced by partial visibility.

To model the complex dependencies among visible clients and historical information, we introduce the attention mechanism, a widely adopted approach for handling variable-length and sequential inputs while capturing interactions among different entities. Building upon this, we propose a \emph{Spatio-Temporal attention} architecture that integrates historical global information with current observations, and solve the POMDP with Deep Q-Learning (DQL).

In this paper, we propose an RL-based framework that adapts federated learning to the time-varying set of visible clients under partial visibility.
Our main contributions are as follows: (1) We formulate federated learning under partial visibility as a POMDP, in which the server adapts its training to an incompletely observed, time-varying set of visible clients, casting adaptation as a sequential decision-making problem (Section~\ref{sec:pomdp}). (2) We develop a multi-step DQL solution with a novel Spatio-Temporal Attention-based Q-Network that jointly encodes variable-length client observations and historical global models, enabling per-client Q-value estimation under dynamic cluster sizes (Sections~\ref{sec:dqn} and~\ref{sec:qnet}). (3) We provide a convergence guarantee for the global model conditioned on a stable Q-network (Section~\ref{sec:convergence}).
Experiments under heterogeneous data distributions show competitive performance across different partial visibility settings.

\section{Related Work}

Reinforcement Learning has attracted increasing attention for client selection in federated learning. In this paradigm, client selection is modeled as a policy that interacts with the environment through the aggregated global model.
Wang~\etal~\cite{favor} propose FAVOR, which formulates client selection as a DQL problem under data heterogeneity and selects the clients with the highest Q values. This work shows the potential of RL in client selection.
FLASH-RL~\cite{flashrl} integrates system heterogeneity, \ie communication delay, into the DQL-based framework and introduces a reputation-based reward, enabling faster FL training. The evaluation of FLASH-RL on different communication channels shows the ability of RL-based client selection to tackle system heterogeneity.
FedAgent~\cite{fedagent} further studies DQL-based client selection. This work addresses the overestimation of Q values in DQL using a Double Deep Q-Network. Knowledge Distillation is integrated to further align the local training of each client.
Other RL frameworks are also studied in FL scenarios. Chen~\etal~\cite{feddrl} and Zhao~\etal~\cite{fedppo} leverage Actor-Critic frameworks and the Proximal Policy Optimization (PPO) strategy to provide continuous aggregation weights for different participating clients.
While these works establish RL-based client selection, most existing works formulate the state of MDPs as indicators of all clients, relying on full visibility.

Partial visibility is characteristic of FL at the IoT/edge tier, where the set of clients visible to the server is dictated by the physical deployment (device availability and server coverage) rather than controlled by the learning system.
The study of this setting is inspired by early work on Decentralized Federated Learning (DFL)~\cite{dfl} and Sequential Federated Learning (SFL)~\cite{sfl,li2023sfl}.
DFL is a paradigm where clients communicate with each other directly to exchange updated parameters. SFL is a related paradigm in which clients update the model in a fixed sequence and pass the updated parameters to the next client.
Yu~\etal~\cite{snake} propose Snake Learning, which extends SFL by introducing a central server. The server communicates and exchanges parameters with each client sequentially while leveraging knowledge distillation to update the global model.
EdgeFLow~\cite{edgeflow} extends the idea that the server migrates among different edges of a network, having access to a group of clients simultaneously rather than one client at a time. The authors also prove the convergence of the training process and propose a client selection strategy based on the norm of gradient differences.
F3AST~\cite{f3ast} studies client selection under intermittent client availability
and selects clients that minimize the impact of client-sampling variance, while FedAWE~\cite{fedawe} equalizes the local training epochs for all participating clients.
Palit~\cite{trust} propose an RL-based client selection method to enhance the defense against malicious clients under partial visibility, showing that RL-based selection remains viable under partial visibility.

Since its proposal in~\cite{attention}, \emph{Attention} mechanisms provide a flexible way to model interactions among elements in a set or sequence by dynamically weighting their relative importance. A large body of research leverages attention mechanisms in the client selection scenario.
FedABC~\cite{fedabc} propose an attention-based algorithm to select clients with relevant semantics and reduce the total selection size. The attention weight is formulated by the KL divergence between the output distributions of updated parameters, which is deterministic instead of learnable.
Chen~\etal~\cite{fedacs} propose FedACS, which measures the similarity of clients through the cosine distance of model parameters and constructs the attention weights upon the similarity matrix.
FedAWAC~\cite{fedawac} studies FL from the perspective of overcoming catastrophic forgetting, which is a phenomenon where the global model parameters adapt to fresh data and deviate from the global optimum. The updated parameters of different clients are re-weighted based on the variance of output logits. The global parameters are averaged with a window of recent global parameters to avoid forgetting.
Training learnable attention networks for client selection remains relatively under-explored, as the server lacks sufficient supervised signals and well-defined objectives for neural network optimization. DQL provides a natural solution by enabling training through replay buffers and temporal-difference losses, which is the motivation of this paper.

\section{Methods}
\subsection{Problem Formulation}
\label{sec:FL}

We consider an FL system consisting of a central server and a population of \(N\) clients, denoted as \(\mathcal{U} = \left\{1, \dots, N\right\}\). Each client possesses a local dataset \(\mathcal{D}_{n}\triangleq \{(\mathbf{x}_i, y_i)\}_{i=1}^{|\mathcal{D}_{n}|}\). In conventional FL systems, \eg FedAvg~\cite{fedavg}, the goal is to minimize the global objective function:
\begin{equation}
    \label{eq:FL}
    \min_{W} \left\{\frac{1}{N}\sum_{n=1}^{N}\mathbb{E}_{\mathbf{x},y\in \mathcal{D}_n}\left[l(\mathbf{x}, y ; W)\right] \right\}
\end{equation}
where \(W \in \mathbb{R}^{d_{\text{model}}}\) denotes the parameters of a neural network, and \(l(\cdot)\) denotes the loss function.
In the \(t^{\text{th}}\) communication round, the server selects a subset of \(K\) clients, denoted as \(\mathcal{S}^t \subseteq \{1, \dots, N\}\), and sends a copy of the current global parameters \(W_{glob}^{t}\) to each selected client. The selected clients run \textit{Stochastic Gradient Descent} (SGD) to obtain the updated parameters \(\{W_{i}^{t}\}_{i\in \mathcal{S}^t}\). The server aggregates the updated parameters according to Eq.~\eqref{eq:Agg}:
\begin{equation}
    \label{eq:Agg}
    W_{glob}^{t+1} = \sum_{i\in \mathcal{S}^t} \frac{|\mathcal{D}_i|}{\sum_{j\in \mathcal{S}^t} |\mathcal{D}_j|}\cdot W_{i}^{t}
\end{equation}

Under partial visibility, the client selection during the \(t^{\text{th}}\) communication round is restricted to currently visible clients, referred to as a \textit{cluster} in this paper, denoted as \(\mathcal{C}^{t} \subset \{1, \dots, N\}\).
Therefore, we have the restricted selection range \(\mathcal{S}^t \subseteq \mathcal{C}^{t}\) where \(\mathcal{C}^{t}\) is externally determined and may vary across communication rounds, making balanced sampling over the entire client population more difficult. Throughout this paper, we assume each communication round admits a non-empty set of visible clients, \ie \(\mathcal{C}^{t} \neq \emptyset\); rounds in which no client is visible are excluded from the training timeline and not counted in the round index \(t\), reflecting practical deployments where the server schedules aggregation only when at least one client is available.

\subsection{POMDP Model}
\label{sec:pomdp}

In this paper, we model the process of client selection and aggregation as a POMDP. A typical POMDP consists of four basic elements: \textit{states}, \textit{actions}, \textit{rewards}, \textit{observations}.
\subsubsection{States}
Most existing MDP-based methods define the states as a combination of local model parameters from all clients and a group of meta-information, including computational capabilities, transmission delay, etc. In this paper, we assume the hardware systems are identical among all clients and thus ignore the meta-information, since our research focuses mainly on partial visibility and data heterogeneity.

Therefore, we define the state of the \(t^{\text{th}}\) round as the local updates that every client \emph{would} produce:
\begin{align}
    \label{eq:state}
    s^{t} &\triangleq \left(W_{1}^{t}, W_{2}^{t}, \dots, W_{N}^{t}\right) \\
    \text{where}\quad & W_{i}^{t} = W_{glob}^{t} - \eta \cdot \nabla_{i} W_{glob}^{t}
\end{align}
where \(\nabla_{i}(\cdot)\) represents the stochastic gradient of the local dataset \(\mathcal{D}_{i}\) and \(\eta\) denotes the learning rate.

Notably, the definition of states is a conceptual description. It does not imply that the server has access to the gradients or local updates of invisible clients, nor that such updates are actually computed. The agent only receives the \textit{observations} in a POMDP.

\subsubsection{Actions}
We define the actions as the selected clients' IDs within the current cluster:
\begin{equation}
    \label{eq:action}
    a^t \triangleq \left\{a_{1}^{t}, a_{2}^{t}, \dots, a_{K}^{t}\right\} \subseteq \mathcal{C}^t
\end{equation}

\subsubsection{Rewards}
The principle of reward design is that the reward directly reflects the performance of the global model after taking an action. In our work, the reward is defined as follows:
\begin{equation}
    \label{eq:reward}
    r^{t} = \lambda\cdot\mathcal{M}(W_{glob}^{t+1}) + (1-\lambda)\cdot r^{t-1}
\end{equation}
where \(\mathcal{M}(\cdot)\) is the F1-score of the global model, which is a widely used metric for reward design in RL-based client selection methods. \(\lambda\) is a hyperparameter that controls the variance of the reward.

The F1-score is evaluated on a small server-side \emph{public} validation set that is disjoint from any client's private data and far smaller than both the training data and the test set. We emphasize that this set is used \emph{only} to measure the global model's performance to orchestrate client selection, never to train the model, so the core privacy guarantee of FL, \ie that no client training sample ever leaves the device, is retained. Using a lightweight server-side performance proxy as the reinforcement signal is standard practice in RL-based federated client selection~\cite{edgeflow,flashrl,fedawac,feddrl,fedabc}. We adopt the classification-oriented F1-score for consistency with prior work; substituting a task-appropriate metric, or an unlabeled or self-supervised proxy that removes the need for any labeled server-side data, generalizes the formulation and is left as future work.

\subsubsection{Observations}
Due to partial visibility, the server is unable to collect the entire state \(s^t\). Instead, it can collect a subset of these local updates, along with the IDs of the visible clients. Therefore, we define the observation for communication round \(t\):
\begin{equation}
    \label{eq:observation}
    o^{t} = \left\{(W_{j}^{t}, j)\right\}, j \in \mathcal{C}^t
\end{equation}
The observation can be interpreted as the state \(s^t\) masked by the visibility \(\mathcal{C}^t\). Therefore, the observation function of POMDP is:
\begin{equation}
    \begin{aligned}
        \label{eq:obsfunc}
        O(a^{t-1}, s^t, o^t) &= \text{Pr}\left(o^t \mid a^{t-1}, s^t\right) \\
        &= \text{Pr}(\mathcal{C}^t)\cdot\prod_{j\in\mathcal{C}^t}\text{Pr}\left(W_j^t\mid W_{glob}^{t-1}\right)
    \end{aligned}
\end{equation}
where \(W_{glob}^{t-1}\) is determined by the previous action \(a^{t-1}\) and the prior state.

\subsection{Deep Q-Learning (DQL) Solution}
\label{sec:dqn}

In this section, we develop a DQL-based agent to solve the POMDP and make client selection decisions under partial visibility.

Unlike many RL problems, FL is characterized by limited samples and irreversible actions, where model updates cannot be rolled back once applied.
Under such constraints, DQL, as an off-policy learning algorithm, is particularly suitable due to its sample efficiency and natural compatibility with discrete decision spaces, making it well aligned with client selection in FL.

\subsubsection{History-based Q-function Approximation}
The canonical approach to POMDP solutions conditions policy decisions upon the cumulative interaction history~\cite{pomdpsurvey}:
\begin{equation}
    \label{eq:history}
    h_{1:t} = \left\{o^{1}, a^{1}, o^{2}, a^{2}, \dots, o^{t}\right\}.
\end{equation}
However, maintaining the full history is computationally impractical.
Instead, we approximate the optimal Q-function using a truncated \emph{temporal context} of the most recent \(H\) rounds:
\begin{equation}
    \label{eq:truncated_history}
    Q(h_{1:t}, a) \approx Q(h_{t-H:t}, a),
\end{equation}
which is assumed to capture sufficient temporal information.

\subsubsection{Client-wise Q-value Decomposition}
At round \(t\), an action \(a^t\) corresponds to selecting a subset of $K$ clients from the visible cluster \(\mathcal{C}^t\).
To handle variable cluster sizes and combinatorial action spaces, we decompose the Q-function into client-wise components:
\begin{equation}
    \label{eq:clientwise_q}
    Q(h_{t-H:t}, a^t) = \frac{1}{K} \sum_{i=1}^{K} \hat{Q}(h_{t-H:t}, a_i^t),
\end{equation}
where \(\hat{Q}(h_{t-H:t}, a_i^t)\) denotes the estimated contribution of selecting client \(a_i^t\) under the current history.
This formulation enables the agent to evaluate each visible client independently and select the top-\(K\) clients with the highest Q-values, naturally supporting clusters with varying sizes.

\subsubsection{Multi-Step DQL Optimization Objective}
The client-wise Q-function \(\hat{Q}(\cdot)\) is parameterized by a deep neural network and trained using multi-step temporal-difference learning.
Given a transition sequence of length \(H\) starting from round \(t\), the multi-step target is defined as:
\begin{equation}
    \label{eq:multistep_target}
    y^t = \sum_{m=0}^{H-1} \gamma^{m} r^{t+m}
    + \gamma^{H} \max_{a' \subseteq \mathcal{C}^{t+H}} Q(h_{t+1:t+H}, a')
\end{equation}
where \(\gamma \in (0,1)\) is the discount factor.
The corresponding optimization objective is given by:
\begin{equation}
    \label{eq:multistep_dqn_loss}
    \mathcal{L}_{\text{DQN}} =
    \mathbb{E}
    \left\|
    y^t - Q(h_{t-H:t}, a^t)
    \right\|^2
\end{equation}

In practice, the maximization over the action space at round \(t+H\) is efficiently implemented by selecting the top-\(K\) clients according to their client-wise Q-values.

\subsubsection{Temporal Parameter Aggregation}
Inspired by FedAWAC~\cite{fedawac}, we average the global parameters over the temporal context window:
\begin{equation}
    \label{eq:tempAgg}
    W_{glob}^{t+1} = \frac{1}{H}\left(\hat{W}_{glob}^{t+1} + \sum_{\tau=0}^{H-2} W_{glob}^{t-\tau}\right)
\end{equation}
where \(\hat{W}_{glob}^{t+1}\) is obtained by Eq.~\eqref{eq:Agg}.

This operation explicitly incorporates historical information into the global parameters, such that the aggregated parameters and the multi-step DQL target \(y^t\) share the same history length \(H\).

\subsection{Spatio-Temporal Attention-based Q-Network}
\label{sec:qnet}
\begin{figure}[!t]
    \centering
    \includegraphics[width=\columnwidth]{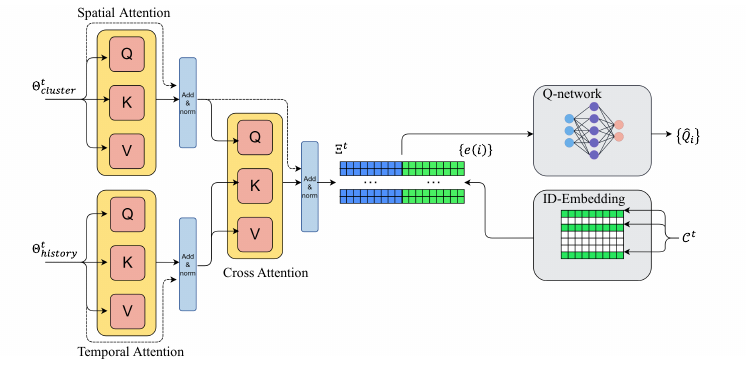}
    \caption{Architecture of the Spatio-Temporal Attention-based Q-Network.}
    \label{fig:network}
\end{figure}

This subsection presents our Spatio-Temporal Attention-based Q-Network, a novel architecture designed to estimate per-client Q-values under partial visibility constraints.

\subsubsection{Representing Temporal Context}
In our POMDP setting, the reward at each communication round is evaluated based on the aggregated global model after the client selection.
As a result, the effect of an action is tightly coupled with the aggregated global parameters.
To capture such temporal dependencies and strengthen the alignment between the reward signal and the Q-network input, we utilize a sequence of historical global parameters \(\left[W_{glob}^{t-H}, W_{glob}^{t-H+1}, \dots, W_{glob}^{t}\right]\) from the current and the most recent \(H\) communication rounds as temporal context.

Specifically, the historical global parameters serve as a compact summary of the recent training trajectory, which is more efficient computationally and contains all the reward-associated information in \(\left[o^{t-H}, a^{t-H}, \dots, o^{t-1}, a^{t-1}\right]\).

Therefore, the input of the Q-network can be represented by the following tuple:
\begin{equation}
    \begin{aligned}
    \hat{Q}\left(h_{t-H:t}, a\right) &\triangleq \hat{Q}\left(W^{t-H:t}_{glob}, o^t, a\right) \\
    &= \hat{Q}\left(W^{t-H:t}_{glob}, \{W_i^t\}_{i\in\mathcal{C}^{t}}, \mathcal{C}^{t}, a\right)
    \end{aligned}
\end{equation}

\subsubsection{Dimension Reduction of model parameters}
Both local and global model parameters are typically high-dimensional, making it computationally expensive to use the original parameters directly as input to the Q-network.

To obtain compact yet informative representations, we map the original model parameters into a low-dimensional feature space before further processing through Random Projection (RP):
\begin{equation}
    \label{eq:rp}
    \text{RP}: \mathbb{R}^{d_{\text{model}}} \mapsto \mathbb{R}^{d_{\text{feat}}},\quad \text{RP}\left(W\right) = \frac{1}{d_{\text{feat}}}\cdot P\cdot W
\end{equation}
where \(P \in \mathbb{R}^{d_{\text{feat}}\times d_{\text{model}}}\) is the projection matrix and its elements are independently sampled from the Standard Gaussian Distribution \(\mathcal{N}(0, 1)\).

A large body of existing works leverages Principal Component Analysis (PCA) as a dimension reduction method~\cite{fedagent,favor,flashrl}. Under partial visibility, however, the visible clients vary over time, making it uncertain for the server to collect sufficient and representative samples to initiate PCA. In contrast, RP possesses similar properties (linearity, near-orthogonality, etc.) without requiring pre-collected samples.

\subsubsection{Spatio-Temporal Attention Architecture}
Given a set of queries $\mathcal{Q}$, a set of keys $\mathcal{K}$, and values $\mathcal{V}$, the output of an attention head is defined as:
\begin{equation}
    \label{eq:attn}
    \text{Attn}\left(\mathcal{Q}, \mathcal{K}, \mathcal{V}\right) = \text{Softmax}\left(\frac{\mathcal{Q}^{T} \times \mathcal{K}}{\sqrt{d_{\text{token}}}}\right) \times \mathcal{V}
\end{equation}
where $\mathcal{Q}, \mathcal{K}, \mathcal{V}$ are obtained by multiplying the inputs with learnable projection matrices respectively.

In our work, each parameter vector of the local updated model \(W_{i}^{t}\) is first compressed via RP to obtain the feature vector \(\omega_{i}^{t}\).
We transform the feature vectors into the input tokens of attention heads \(\theta_{i}^{t}\) through a lightweight nonlinear encoder, denoted as \(f_{\text{enc}}(\cdot): \mathbb{R}^{d_{\text{feat}}}\mapsto \mathbb{R}^{d_{\text{token}}}, \theta_{i}^{t} = f_{\text{enc}}\left(\omega_{i}^{t}\right)\), to extract a semantic representation of the parameters and normalize their scales.

Based on the encoded tokens, we construct a Spatio-Temporal attention architecture to jointly capture inter-client relationships and temporal context. The overall architecture of the network is illustrated in Fig.~\ref{fig:network}.

At each communication round, a \emph{Spatial Attention} module is applied over the set of visible client tokens, allowing the Q-network to reason about the relative contribution of different clients within the same cluster.
This Spatial Attention module is permutation-invariant, so its output does not depend on the order of clients in the cluster.

To incorporate temporal context, the encoded tokens of the most recent global parameters are processed by a \emph{Temporal Attention} module.
This module enables the network to capture temporal dependencies and the relationship between varying rewards and parameter differences.
The Temporal Attention module is implemented with Positional Encoding and Causal Mask to learn the order of the tokens.

Finally, the processed spatial and temporal tokens are integrated through cross-attention, where client representations attend to the historical global context.
This design allows the Q-network to evaluate each client based on the temporal context, resulting in context-aware client features for subsequent Q-value estimation.

We add residual connections to each attention head for stable gradient back-propagation and a consistent latent space. The overall function of attention layers is formulated as:
\begin{align}
    \label{eq:attnlayers}
    \Xi^{t} &= \Theta_{\text{Spatial}}^{t} + \text{Attn}\left(\Theta_{\text{Spatial}}^{t}, \Theta_{\text{Temporal}}^{t}, \Theta_{\text{Temporal}}^{t}\right) \\
    \text{where} &\quad \Theta_{\text{Spatial}}^{t} = \Theta_{\text{cluster}}^{t} + \text{Self-Attn}\left(\Theta_{\text{cluster}}^{t}\right) \\
    \text{and} &\quad \Theta_{\text{Temporal}}^{t} = \Theta_{\text{history}}^{t} + \text{Self-Attn}\left(\Theta_{\text{history}}^{t}\right)
\end{align}
where the input \(\Theta_{\text{cluster}}^{t} = \left[\theta_{1}^{t}, \dots, \theta_{|\mathcal{C}^{t}|}^{t}\right]\) and \(\Theta_{\text{history}}^{t} = \left[\theta_{glob}^{t-H}, \dots, \theta_{glob}^{t}\right]\). The function \(\text{Self-Attn}(\cdot)\) equals \(\text{Attn}(\cdot,\cdot,\cdot)\).

\subsubsection{Identity-Aware Embedding and Memory}
Intermittent client participation under partial visibility induces temporal sparsity in observation trajectories, making it hard to track each client's contribution over time. To overcome this, we propose an Identity-Aware Embedding module that generates unique client identifiers via trainable identity vectors, enabling cross-round impact modeling despite intermittent participation. It is defined as:
\begin{equation}
    \label{eq:idemb}
    e(\cdot):\mathbb{N}\mapsto \mathbb{R}^{d_{\text{emb}}}
\end{equation}

These ID embeddings provide client-level memory, enabling the Q-network to capture persistent, client-specific characteristics across rounds.

\subsubsection{Dueling Q-Value Decomposition}
Finally, the Q-network adopts a dueling architecture to separately estimate the global value of the current training state and the relative advantage of selecting each client:
\begin{equation}
    \hat{Q}_{i}^{t} = \operatorname{Val}\left(\bar{\xi}^{t}, \bar{e}\right) + \operatorname{Adv}\left(\xi_{i}^{t}, e(i)\right) - \bar{\operatorname{Adv}}
\end{equation}
where \(\xi_{i}^{t}\) denotes the \(i^{\text{th}}\) vector of \(\Xi^{t}\), and variables with a bar denote the average over \(i\in\mathcal{C}^{t}\).

\subsection{Algorithm}
In this section, we present the workflow to integrate the proposed POMDP-based client selection strategy into the FL training process. To stabilize Q-value estimation, the $\varepsilon$-greedy action strategy and the Double-DQN architecture are widely utilized in RL applications~\cite{favor,fedagent}. The pseudocode of our proposed method is shown in Algorithm~\ref{alg:flow}.
\begin{algorithm}[!t]
\caption{FL client selection under partial visibility}
\label{alg:flow}
\begin{algorithmic}[1]
\REQUIRE Total rounds $T$, client set $\mathcal{U}$, selection size $K$, discount factor $\gamma$, soft-update coefficient $\tau$, exploration decay rate $e$
\STATE Initialize global model weights $W^0_{glob}$
\STATE Initialize History Buffer $\mathcal{H} \leftarrow \emptyset$ with maximum length $H$
\STATE Initialize Replay Buffer $\mathcal{B} \leftarrow \emptyset$ with maximum capacity $B$
\STATE Initialize Online Q-network $Q_\theta$ random weights $\theta$ and Target Q-network with $\bar{\theta} \leftarrow \theta$
\FOR{round $t = 0, 1, \dots, T-1$}
    \STATE \textbf{Local Update:} Each client $i \in \mathcal{C}^t$ downloads $W^t_{glob}$ and performs $k$ rounds of local SGD:
    \[W_i^t = W_{glob}^t - \eta \sum_{j=0}^{k-1} \nabla \ell(W^t_{i,j}; \mathcal{D}_i)\]
    \STATE \textbf{Observation and History Construction:} Build the Q-network input $s^t = (\mathcal{C}^t, \{\text{RP}(W_i^t)\}_{i \in \mathcal{C}^t}, \text{RP}(W^t_{glob}), \mathcal{H})$

    \STATE \textbf{$\mathbf{\varepsilon}$-Greedy Client Selection:} With $\varepsilon_t = \max(0.1, 1 - e \cdot t)$ and $p = 1 - \varepsilon_t$:
    \STATE \hskip\algorithmicindent with probability $p$: \hfill \COMMENT{Greedy selection}
    \STATE \hskip\algorithmicindent $\mathcal{S}^t = \arg\!\mathop{\mathrm{TopK}}_{i \in \mathcal{C}^t} Q_\theta(s^t)$
    \STATE \hskip\algorithmicindent with probability $1 - p$: \hfill \COMMENT{Exploration}
    \STATE \hskip\algorithmicindent $\mathcal{S}^t = \mathrm{Random}(\mathcal{C}^t, K)$

    \STATE \textbf{Aggregation:}
    \STATE Compute weighted average of selected clients' models:
    \[\hat{W}^{t+1}_{glob} = \frac{1}{\sum_{i \in \mathcal{S}^t} |\mathcal{D}_i|}\sum_{i \in \mathcal{S}^t} |\mathcal{D}_i| \, W_i^t\]
    \STATE Store $\hat{W}^{t+1}_{glob}$ into $\mathcal{H}$ and calculate $W^{t+1}_{glob}$ according to Eq.~\eqref{eq:tempAgg}
    \STATE \textbf{Reward:} Evaluate $W_{glob}^{t+1}$ and compute the reward $r^t$ according to Eq.~\eqref{eq:reward}
    \STATE \textbf{Agent Training:} Update $Q_\theta$ via Algorithm~\ref{alg:agent_training}

\ENDFOR
\end{algorithmic}
\end{algorithm}

\begin{algorithm}[!t]
\caption{Agent Training}
\label{alg:agent_training}
\begin{algorithmic}[1]
\REQUIRE Replay Buffer $\mathcal{B}$, Online Q-network $Q_\theta$, Target Q-network $Q_{\bar{\theta}}$, discount factor $\gamma$, soft-update coefficient $\tau$, training iterations $b$
\STATE Store transition $(s^t, \mathcal{S}^t, r^t)$ into $\mathcal{B}$
\FOR{$\text{iter} = 0, 1, \dots, b-1$}
\STATE Sample a mini-batch of $H$-step transitions from $\mathcal{B}$
\STATE Update Online Q-network parameters $\theta$ by minimizing $\mathcal{L}_{\text{DQN}}$ via SGD
\STATE Soft-update Target Q-network: $\bar{\theta} \leftarrow \tau \theta + (1 - \tau) \bar{\theta}$
\ENDFOR
\end{algorithmic}
\end{algorithm}
\subsubsection{$\varepsilon$-Greedy Action}
To ensure sufficient coverage of the action space for the Q-network, we employ a linearly decaying exploration probability $\varepsilon_t = \max(0.1, 1 - e \cdot t)$. In the early stages of training, a larger $\varepsilon_t$ promotes extensive exploration of the action space, allowing the Q-network to gather information about all clients. As training progresses, $\varepsilon_t$ decays to prioritize exploitation, which limits the accuracy loss from selecting suboptimal clients.
\subsubsection{Double-DQN}
The Double DQN (DDQN) architecture is widely applied in DQL tasks to mitigate the overestimation of \(Q\) caused by using the same Q-network for action selection and \(Q\) estimation~\cite{ddqn}.
The DDQN architecture maintains an online network parameterized by $\theta$ for action selection, and a target network parameterized by $\bar{\theta}$ to provide Q-value estimation.
The target network parameters $\bar{\theta}$ are softly updated with a smoothing factor $\tau \in (0,1)$, so that they track the online parameters $\theta$, which enhances training stability.

\section{Convergence Analysis}\label{sec:convergence}
In this section, we analyze the convergence of the global model under the proposed client selection method. A complete convergence guarantee for DQL remains an open problem~\cite{theodql,zhang2023convergence}. Therefore, we provide convergence guarantees for the global model conditioned on a fixed, stable Q-network. Under this condition, the client selection policy induced by the Q-network is treated as deterministic across communication rounds.

We first assume that both the FL objective function and the Q-function are Lipschitz-smooth (Assumption~\ref{ass:lsmooth_flow}), and that the data heterogeneity and the variance of stochastic gradients are bounded (Assumptions~\ref{ass:gradient_flow}--\ref{ass:niid_flow}). These are standard assumptions widely adopted in FL convergence analyses. We treat the Q-network as a function of the gradients and assume that the Q-function satisfies $\mu$-Quadratic Growth (Assumption~\ref{ass:qstrong}).
Based on these assumptions, we establish Theorem~\ref{th:conv_sel_flow}, which guarantees the convergence of the global model under the proposed selection strategy.


\begin{assumption}[$L$-Smoothness] \label{ass:lsmooth_flow}
    \textbf{(i)} The objective function $F(\cdot)$ is $L$-smooth, \ie there exists a Lipschitz constant $L$ such that for any $W, W'$:
    \begin{equation}
        F(W) - F(W') \leq \inner{\nabla F(W')}{W-W'} + \frac{L}{2} \normsq{W-W'}
    \end{equation}
    \textbf{(ii)} The Q-function $Q(\cdot)$ is $L_Q$-smooth, \ie $\nabla Q$ is $L_Q$-Lipschitz continuous:
    \begin{equation}
        Q(y) \geq Q(x) + \inner{\nabla Q(x)}{y-x} - \frac{L_Q}{2}\normsq{y-x}
    \end{equation}
\end{assumption}

\begin{assumption}[Bounded Stochastic Gradient] \label{ass:gradient_flow}
    The norm and variance of the stochastic gradients $\tilde{g}_{n,j}^t$ are bounded by $G^2$ and $\sigma^2$, respectively:
    \begin{align}
        &\E{\normsq{\tilde{g}_{n,j}^t}} \leq G^2 \\
        &\E{\normsq{\tilde{g}_{n,j}^t - \E{\tilde{g}_{n,j}^t}}} \leq \sigma^2
    \end{align}
\end{assumption}

\begin{assumption}[Data Heterogeneity] \label{ass:niid_flow}
    The cluster-level objective is defined as $F_{\mathcal{C}^t}(\cdot)\triangleq \frac{1}{|\mathcal{C}^t|}\sum_{n\in\mathcal{C}^t} F_n(\cdot)$. We bound both the inter-cluster and intra-cluster heterogeneity: \\
    \textbf{(i)} The discrepancy between the global gradient and each cluster's gradient is bounded by $d_{\mathcal{C}^t}^2$:
    \begin{equation}
        \E{\normsq{\nabla F(W_{glob}^t) - \nabla F_{\mathcal{C}^t}(W_{glob}^t)}} \leq d_{\mathcal{C}^t}^2
    \end{equation}
    \textbf{(ii)} The intra-cluster gradient variance is bounded by $\Gamma_{\mathcal{C}^t}^2$:
    \begin{equation}
        \mathbb{E}_{n \in \mathcal{C}^t}\left[{\normsq{\nabla F_n(W_{glob}^t) - \nabla F_{\mathcal{C}^t}(W_{glob}^t)}}\right] \leq \Gamma_{\mathcal{C}^t}^2
    \end{equation}
\end{assumption}

\begin{assumption}[Quadratic Growth of $Q$] \label{ass:qstrong}
    For any $x$:
    \begin{equation}
        \normsq{x - v^*} \leq \frac{2}{\mu}\left(Q(v^*) - Q(x)\right)
    \end{equation}
    where $v^* = \arg\max_v Q(v)$ and $\mu > 0$ is the quadratic growth constant.
\end{assumption}

\begin{remark}[On Assumption~\ref{ass:qstrong}]
    Assumption~\ref{ass:qstrong} implies the existence of a maximum for the Q-function, which can be ensured in practice by clipping the Q-network output, a common practice in DQL. The assumption only constrains the sharpness of the
    Q-function and is substantially weaker than requiring \(\mu\)-strong
    concavity, serving as a commonly adopted condition in convergence
    analyses.
\end{remark}


\begin{theorem}[Convergence with DQL-based Selection] \label{th:conv_sel_flow}
    For client selection ratio $c = |\mathcal{S}^t| / |\mathcal{C}^t|$, learning rate $\eta$ satisfying $Lk\eta < 1$, under Assumptions~\ref{ass:lsmooth_flow}--\ref{ass:qstrong}:
    \begin{equation}
        \begin{split}
            \frac{1}{T}\sum_{t=0}^{T-1}\E{\normsq{\nabla F(W_{glob}^t)}} \leq{} &\frac{4}{k\eta T}\left(F(W_{glob}^0) - F^*\right) \\
            &+ \mathbb{C} + \frac{4L^2k^2\eta^2G^2}{3}
        \end{split}
        \label{eq:conv_sel_flow}
    \end{equation}
    where $\mathbb{C} = 4\bar{d}^2 + \frac{8L_Q}{\mu c}\bar{\Gamma}^2 + \frac{16(L_Q + \mu c)}{\mu^2 c}\bar{\Delta}_Q^t + \frac{2L\eta\sigma^2}{c\bar{|\mathcal{C}|}}$, with time-averaged quantities $\bar{X} \triangleq \frac{1}{T}\sum_{t=0}^{T-1} X_t$ (\eg $\bar{d}^2$, $\bar{\Gamma}^2$, $\bar{|\mathcal{C}|}$), and $\Delta_Q^t \triangleq Q^* - Q(\nabla F_{\mathcal{C}^t}) \geq 0$ is the gap between Q's maximum value and the value of the current cluster.
\end{theorem}

\begin{remark}
    In the bound of Theorem~\ref{th:conv_sel_flow}, the term \(\mathbb{C}\) captures the error induced by varying visibility across communication rounds, appearing in time-averaged form. \(\mathbb{C}\) decreases when the
    data distributions within each cluster are more homogeneous and remain
    consistent across rounds, \ie when each cluster is approximately
    representative of the global population. A smaller \(\mathbb{C}\)
    tightens the gradient bound, implying that the global objective
    converges more closely to the optimum, consistent with the intuition that
    partial visibility has limited impact when clusters are well-balanced.
\end{remark}

The complete proofs are provided in Appendix~\ref{app:conv_sel_flow}.

\section{Experimental Analysis}
\subsection{Datasets}
We conduct experiments on three representative datasets:
\begin{itemize}
    \item \textbf{CIFAR-10}~\cite{cifar10}: A standard image classification dataset containing 60,000 RGB images in 10 classes, with 50,000 training images and 10,000 test images.
    \item \textbf{Fashion-MNIST}~\cite{fashionmnist}: A dataset of grayscale images, consisting of 60,000 training images and 10,000 test images across 10 fashion categories.
    \item \textbf{UCI-HAR}~\cite{uci}: A dataset of accelerometer signals collected from smartphones of 30 subjects. The target is to detect human activities, \eg walking and sitting.
\end{itemize}

\subsection{Data Heterogeneity Simulation}
To evaluate the robustness of our method against severe data heterogeneity, we construct the non-IID distributions of local datasets with the following strategies:
\begin{itemize}
    \item \textbf{Dirichlet Distribution}: A widely used non-IID setting, partitioning the datasets among clients using a Dirichlet distribution with concentration parameter \(\alpha\). We set \(\alpha=0.1\) in our experiments to create heterogeneous distributions.
    \item \textbf{Label Skew Distribution}: We assign each client a subset of classes, ensuring that each client only has data from exactly two classes. This simulates scenarios where each client has access to extremely limited types of data.
\end{itemize}

\subsection{Partial Visibility Settings}
To simulate real-world partial visibility scenarios, we design two visibility patterns:
\begin{itemize}
    \item \textbf{Mobile Server (MS)}: The server moves across different network regions (Fig.~\ref{subfig:mobile}). Clients are grouped into clusters, and the server randomly connects to one of the clusters at each communication round.
    \item \textbf{Random Availability (RA)}: All clients have an identical probability \(p\) of being available at each communication round (Fig.~\ref{subfig:random}). The selection size is \(\min(|\mathcal{C}^t|, K)\) under this setting.
\end{itemize}

\subsection{Baselines}
We compare our method with the following baselines:
\begin{itemize}
    \item \textbf{FedProx}~\cite{fedprox}: A widely used federated learning algorithm that introduces a proximal term to the local objective to mitigate the impact of data heterogeneity.
    \item \textbf{HA-EdgeFLow}~\cite{edgeflow}: A method designed for federated learning with a mobile server, which selects clients based on the norm of gradient differences.\footnote{https://github.com/hqj-les30/HA-EdgeFLow}.
    \item \textbf{FedAWAC}~\cite{fedawac}: A method designed for overcoming catastrophic forgetting in FL, which assigns different weights to clients based on the variance of the logits predicted by local models and averages the global models of recent communication rounds for aggregation.
    \item \textbf{F3AST}~\cite{f3ast}: A method designed to tackle partial visibility, which adaptively adjusts selection probabilities based on visibility statistics to minimize sampling variance asymptotically.\footnote{https://github.com/mriberodiaz/f3ast}.
\end{itemize}

\subsection{Implementation Details}
All experiments are implemented in PyTorch 2.4.0. For the CIFAR-10 and Fashion-MNIST datasets, we set \(N = 100\), and each client has the same local dataset size. For the UCI-HAR dataset, we consider each subject as a client such that \(N=21\) and the local data is naturally non-IID.
Each selected client performs local training using SGD with a learning rate of 0.001 and a batch size of 64 for 3 local epochs. For the DQN agent, the transitions are cached in a replay buffer with a length of 600. We set the discount factor to \(\gamma=0.9\). We adopt the Double-DQN strategy, where the soft-update rate is set to \(0.005\). Unless otherwise specified, models are trained for 1,500 communication rounds on CIFAR-10 and 600 rounds on Fashion-MNIST, with the client selection size \(K\) fixed to 5 by default.
In the MS scenario, each cluster contains 10 clients, while in the RA scenario, each client is visible with a probability of \(p=0.1\). Our implementation is publicly available at \url{https://github.com/hqj-les30/spattn}.

\begin{table*}[!tb]
    \centering
    \begin{tabular}{@{}lllrrrrr@{}}
    \toprule
Dataset & Visibility & Heterogeneity & FedProx & HA-EdgeFLow & FedAWAC & F3AST & Ours \\ \midrule
CIFAR-10 & MS & Dirichlet & 65.20\(\pm\)7.20 & 65.96\(\pm\)5.29 & 74.07\(\pm\)3.56 & 70.77\(\pm\)4.76 & \textbf{76.16\(\pm\)1.76} \\
CIFAR-10 & MS & Label Skew & 41.46\(\pm\)5.75 & 37.10\(\pm\)6.00 & 44.99\(\pm\)5.23 & 40.15\(\pm\)5.87 & \textbf{50.99\(\pm\)3.20} \\
CIFAR-10 & RA & Dirichlet & 62.84\(\pm\)7.61 & 62.97\(\pm\)7.73 & 73.71\(\pm\)3.64 & 70.85\(\pm\)5.69 & \textbf{75.13\(\pm\)2.65} \\
CIFAR-10 & RA & Label Skew & 40.74\(\pm\)4.95 & 39.98\(\pm\)5.75 & 44.82\(\pm\)4.80 & 39.30\(\pm\)5.48 & \textbf{52.60\(\pm\)3.24} \\
Fashion & MS & Dirichlet & 79.29\(\pm\)5.82 & 78.28\(\pm\)6.25 & 83.39\(\pm\)3.51 & 82.72\(\pm\)3.54 & \textbf{85.21\(\pm\)1.91} \\
Fashion & MS & Label Skew & 62.76\(\pm\)8.92 & 56.90\(\pm\)6.28 & 63.13\(\pm\)7.97 & 63.81\(\pm\)8.18 & \textbf{69.83\(\pm\)7.01} \\
Fashion & RA & Dirichlet & 79.15\(\pm\)5.38 & 80.34\(\pm\)4.96 & \textbf{84.58\(\pm\)3.10} & 81.94\(\pm\)4.09 & 83.87\(\pm\)2.95 \\
Fashion & RA & Label Skew & 62.55\(\pm\)8.81 & 56.12\(\pm\)6.86 & 63.75\(\pm\)8.41 & 61.40\(\pm\)9.04 & \textbf{65.33\(\pm\)6.07} \\
UCI-HAR & MS &  & 91.83\(\pm\)2.67 & 89.76\(\pm\)1.24 & \textbf{91.90\(\pm\)2.39} & 90.40\(\pm\)2.53 & 90.95\(\pm\)1.73 \\
UCI-HAR & RA &  & 92.11\(\pm\)1.74 & 89.90\(\pm\)2.91 & \textbf{92.38\(\pm\)1.97} & 90.25\(\pm\)2.22 & 90.73\(\pm\)1.29 \\
 \bottomrule
    \end{tabular}
    \vspace{0.5em}
    \caption{Test Accuracy of Different Methods Under Different Settings.\label{tab:acc}}
    \vspace{-1.0em}
    \footnotesize
    Results are reported as mean\(\pm\)std obtained from 5 independent runs with different random seeds. Best results are shown in \textbf{bold}. For each seed, we reported the average accuracy of the last 50 communication rounds.
\end{table*}

\begin{figure*}[!t]
    \centering
    \includegraphics[width=0.85\textwidth]{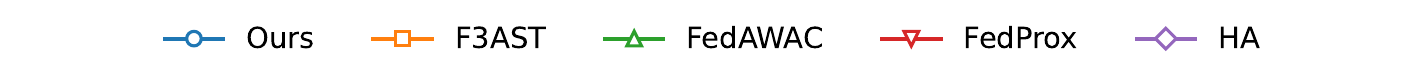}
    \vspace{-0.5em} 

    \subfloat[Fashion, Dirichlet, MS]{%
        \includegraphics[width=0.23\textwidth]{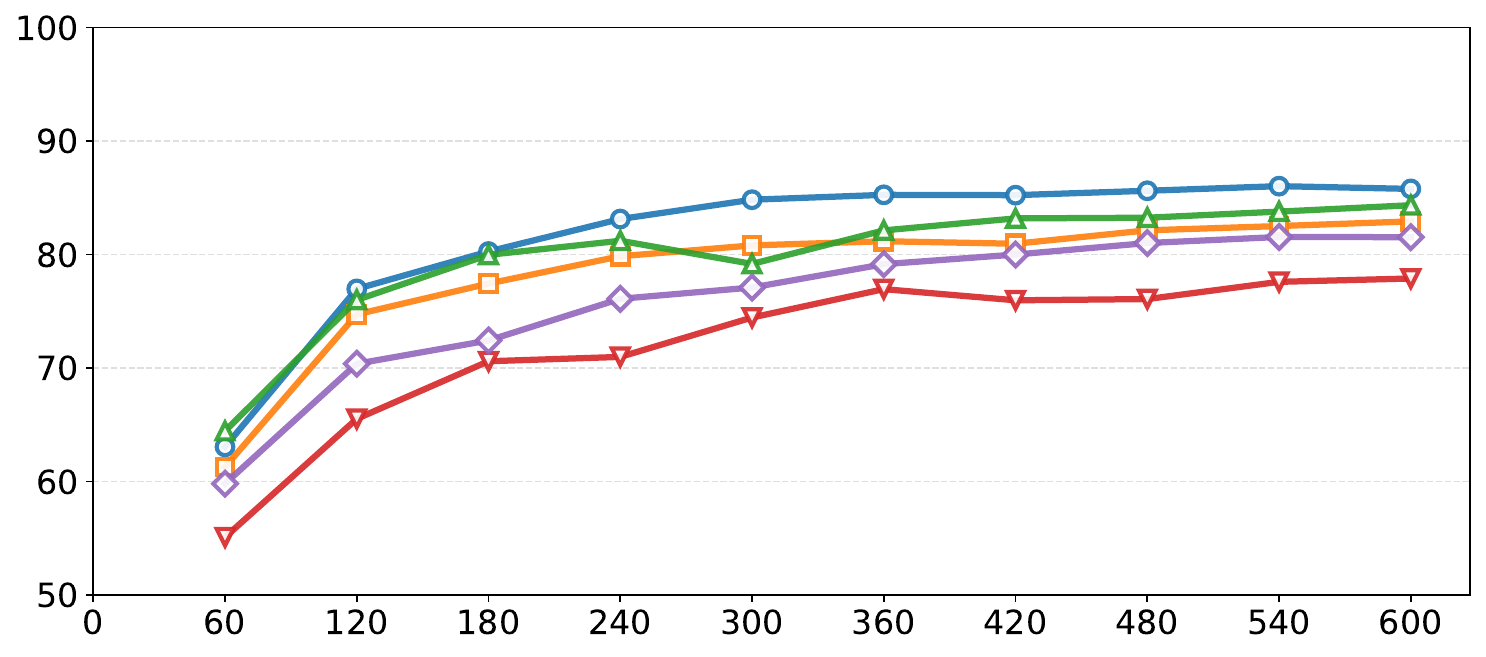}%
        \label{subfig:clsfmdiri}%
    }
    \hfil
    \subfloat[CIFAR-10, Dirichlet, MS]{%
        \includegraphics[width=0.23\textwidth]{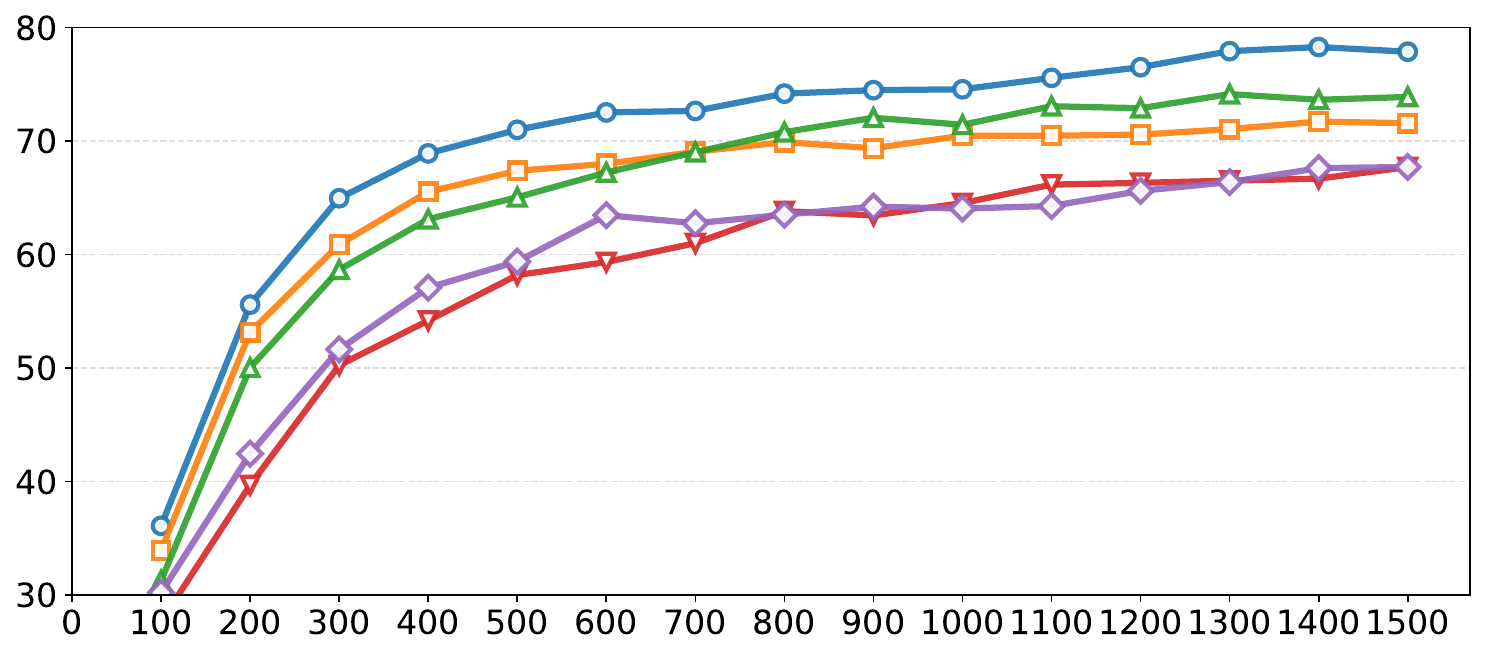}%
        \label{subfig:clscfdiri}%
    }
    \hfil
    \subfloat[Fashion, Label~Skew, MS]{%
        \includegraphics[width=0.23\textwidth]{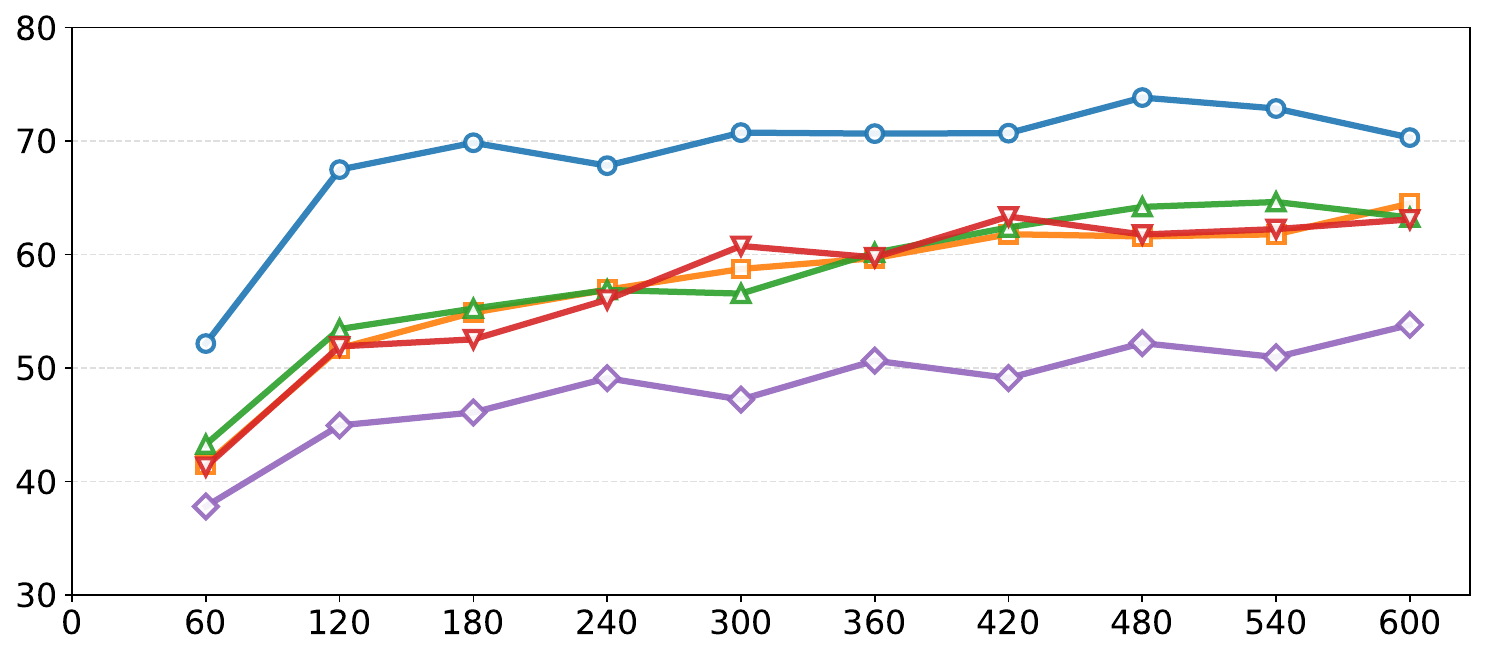}%
        \label{subfig:clsfmskew}%
    }
    \hfil
    \subfloat[CIFAR-10, Label~Skew, MS]{%
        \includegraphics[width=0.23\textwidth]{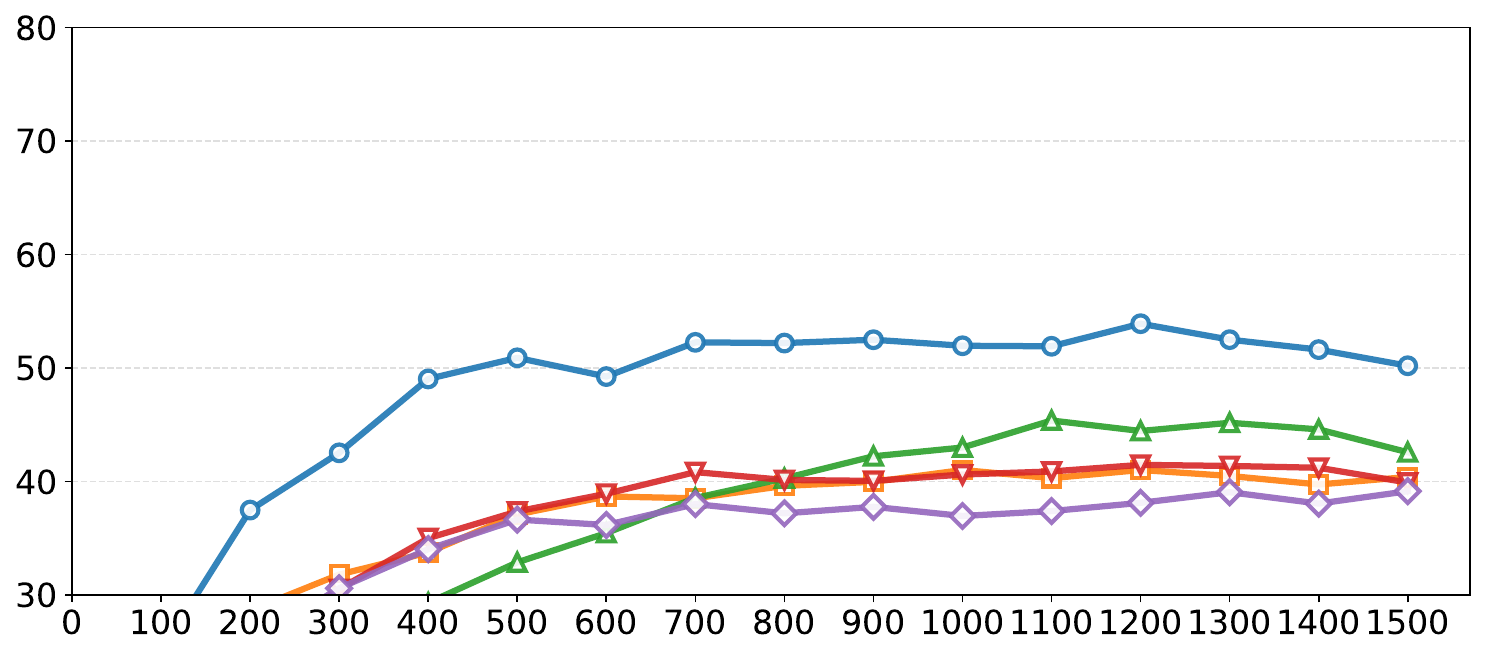}%
        \label{subfig:clscfskew}%
    }

    \subfloat[Fashion, Dirichlet, RA]{%
        \includegraphics[width=0.23\textwidth]{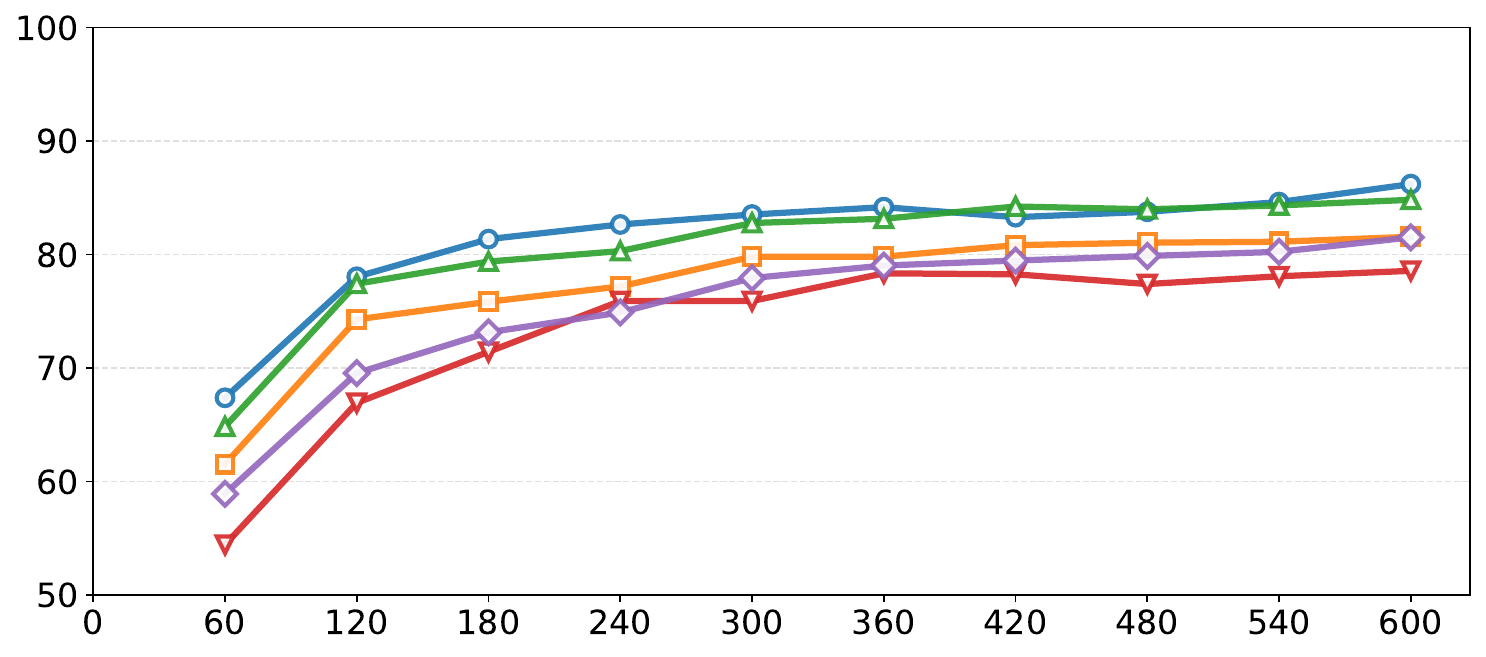}%
        \label{subfig:rafmdiri}%
    }
    \hfil
    \subfloat[CIFAR-10, Dirichlet, RA]{%
        \includegraphics[width=0.23\textwidth]{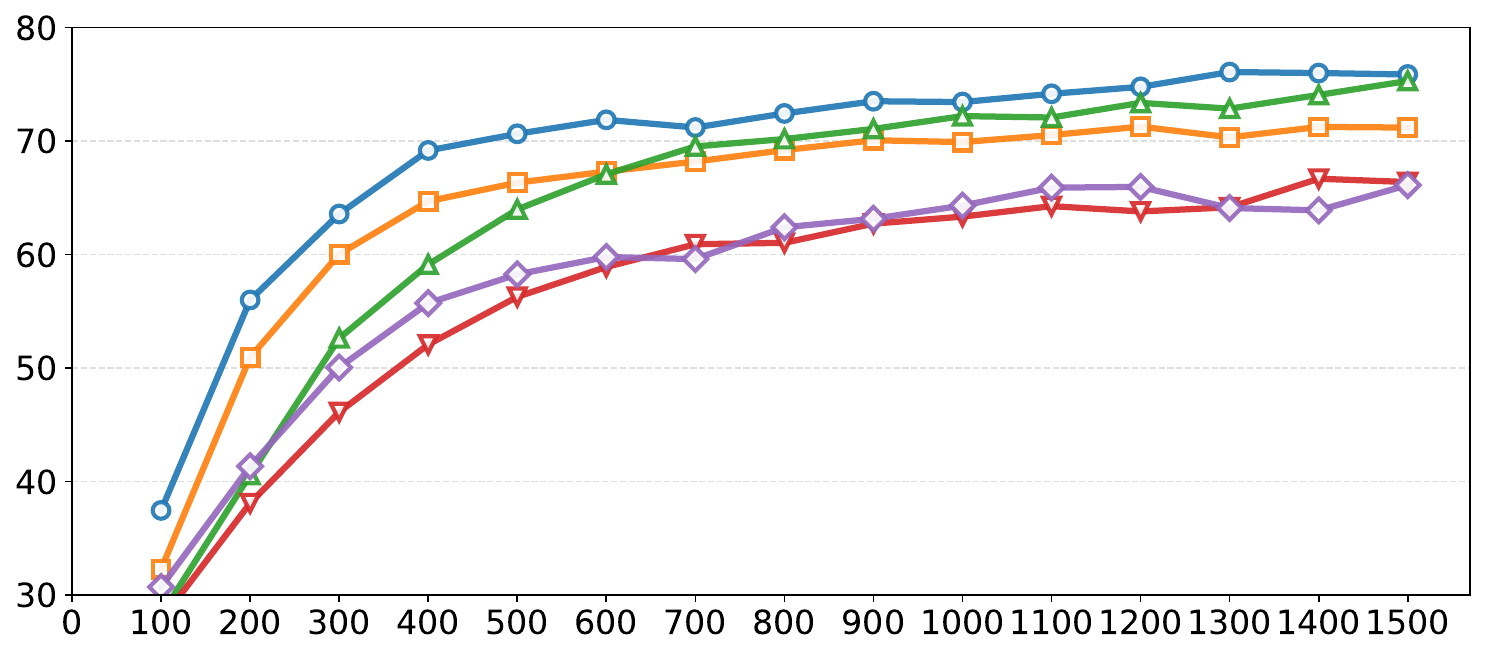}%
        \label{subfig:racfdiri}%
    }
    \hfil
    \subfloat[Fashion, Label~Skew, RA]{%
        \includegraphics[width=0.23\textwidth]{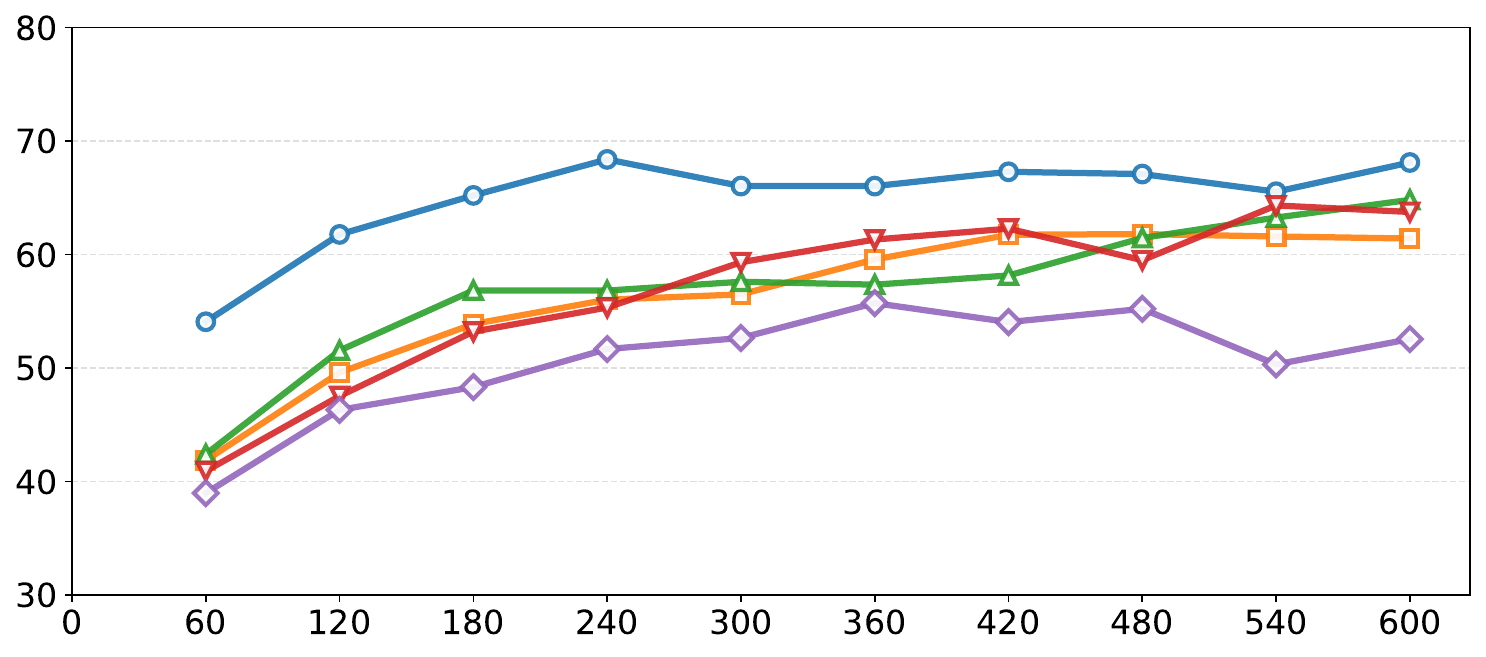}%
        \label{subfig:rafmskew}%
    }
    \hfil
    \subfloat[CIFAR-10, Label~Skew, RA]{%
        \includegraphics[width=0.23\textwidth]{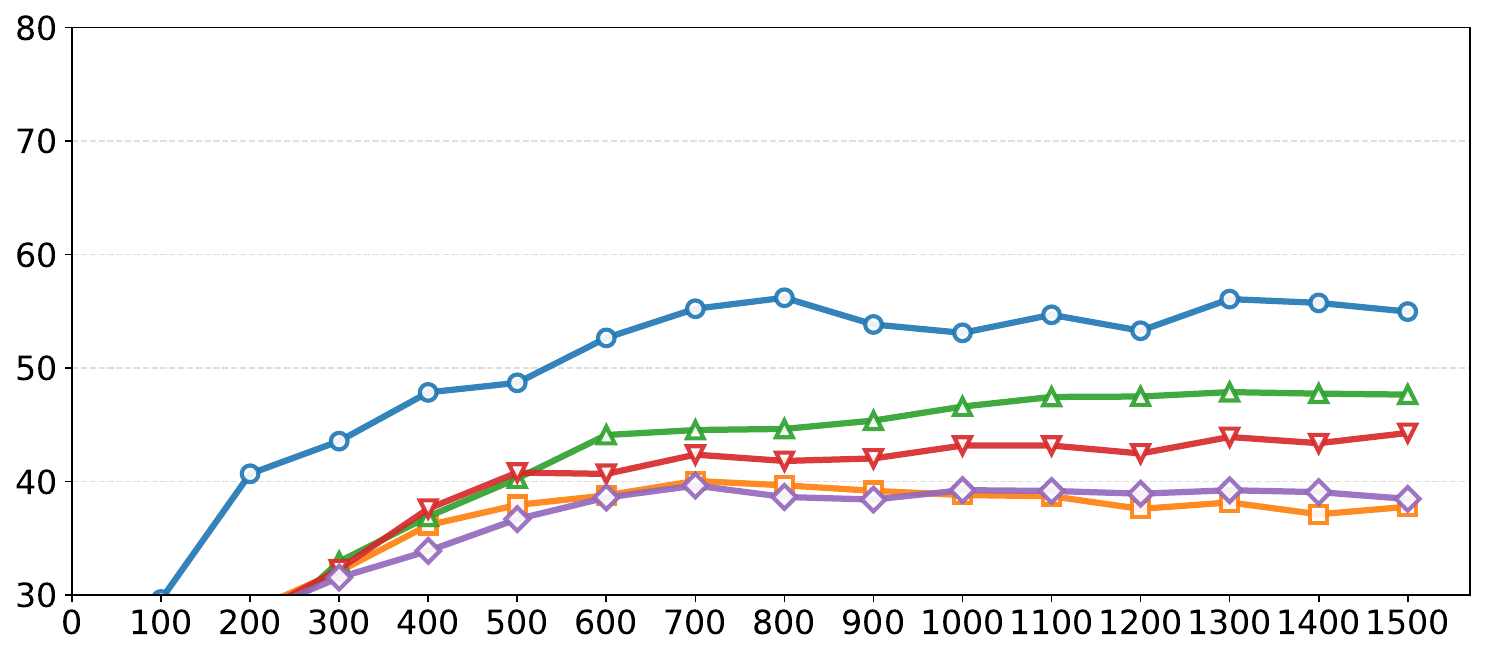}%
        \label{subfig:racfskew}%
    }
    \caption{Accuracy versus communication rounds under various settings.\label{fig:baseline_curve}}
\end{figure*}

\begin{figure}[!t]
    \centering

    \includegraphics[width=0.8\columnwidth]{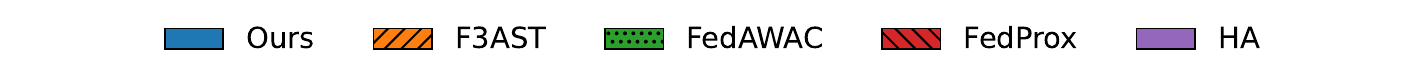}
    \vspace{-0.2em} 

    \subfloat[Dirichlet]{%
        \includegraphics[width=0.9\columnwidth]{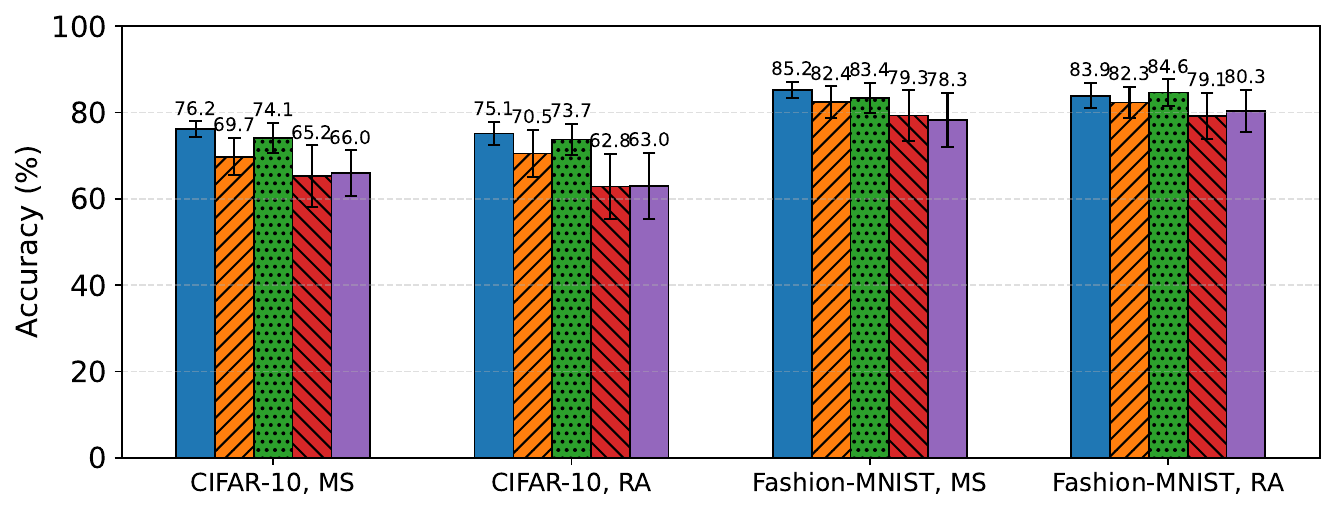}%
        \label{subfig:pillar1}%
    }\\ 
    \vspace{-0.3em} 

    \subfloat[Label-Skew]{%
        \includegraphics[width=0.9\columnwidth]{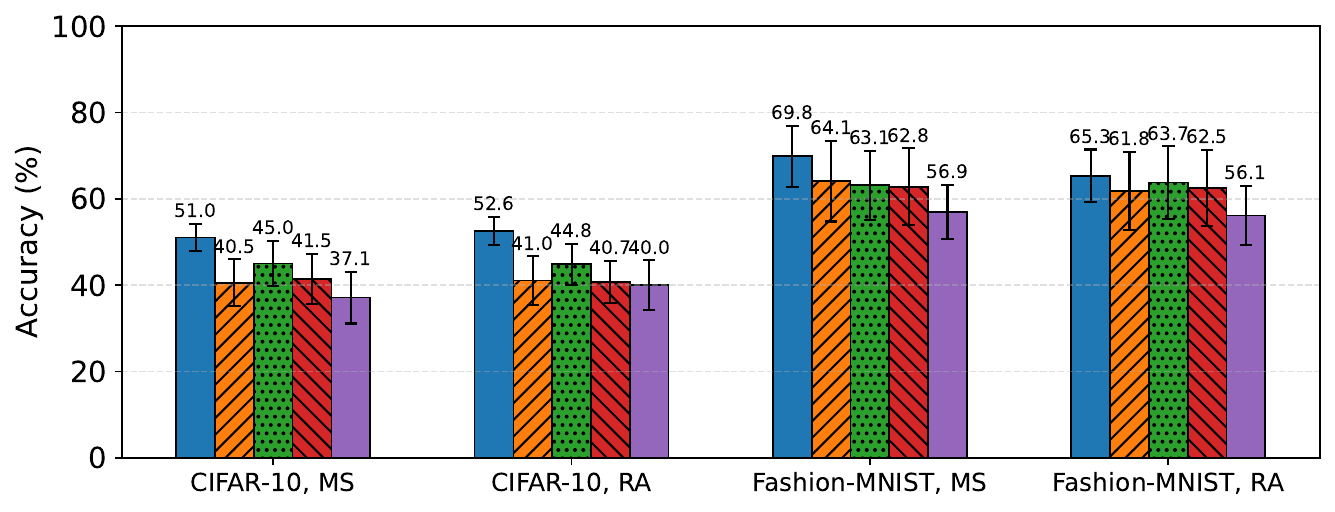}%
        \label{subfig:pillar2}%
    }\\ 
    \vspace{-0.3em}

    \subfloat[UCI-HAR]{%
        \includegraphics[width=0.9\columnwidth]{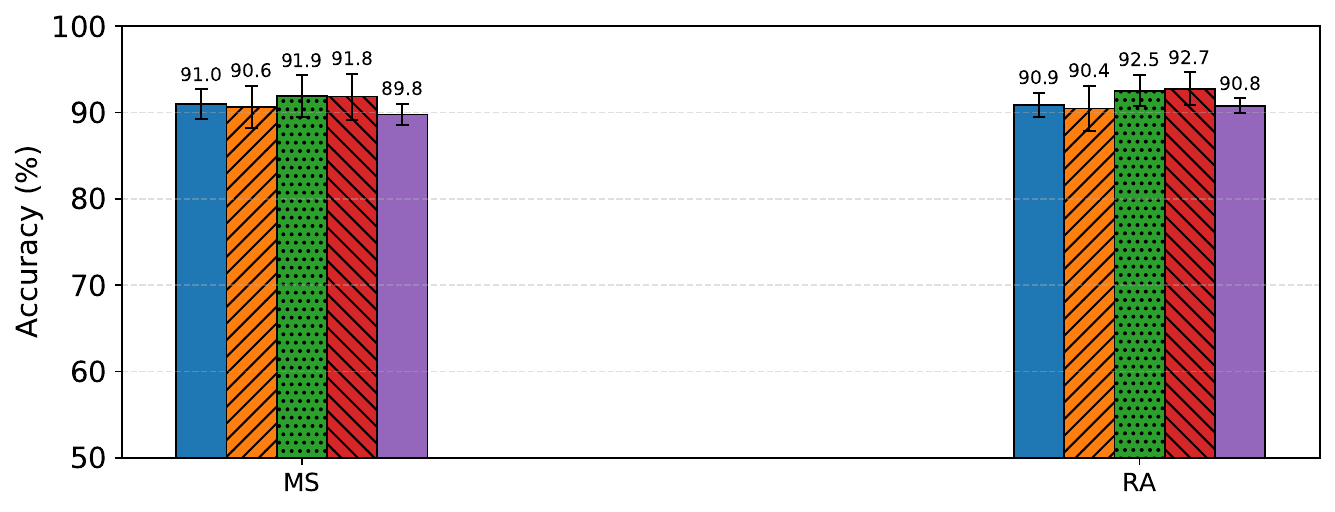}%
        \label{subfig:pillar3}%
    }

    \caption{Final accuracy. Error bars reflect the standard deviation over last 50 communication rounds.\label{fig:pillar}}
\end{figure}

\subsection{Accuracy Performance}

Under identical partial visibility settings, we benchmark our method against representative baselines.

As quantified in Table~\ref{tab:acc}, our proposed method achieves competitive or superior accuracy across all evaluated scenarios, while demonstrating relatively low standard deviations over multiple runs, confirming its enhanced robustness.
Our method achieves the highest test accuracy in 7 out of 10 evaluated scenarios, with an average improvement of \(3.8\%\) over the strongest baseline across these seven scenarios. Under the Label Skew setting, our method consistently ranks first, raising the final accuracy by \(5.3\%\) on average. In the remaining scenarios, our method achieves competitive results, falling within \(1.9\%\) of the best-performing baseline.

Fig.~\ref{fig:baseline_curve} demonstrates the accuracy evolution of different methods under various partial visibility and heterogeneity settings.
Under identical heterogeneity, the accuracy evolution trend of each method remains consistent across both visibility scenarios.
Our method shows a clear gain over the baselines in the Label Skew scenarios (Fig.~\ref{fig:baseline_curve}(c), (d), (g), (h)), with faster accuracy improvement and higher accuracy after convergence.
Under the Dirichlet distribution, although the accuracy difference among methods is relatively moderate, our method still achieves competitive or superior accuracy.
In addition, the performance gains of our method are more pronounced in the MS scenario than in the RA scenario.

Fig.~\ref{fig:pillar} demonstrates the final accuracy of each method, as well as the standard deviation over the last 50 communication rounds.
Under the more heterogeneous Label Skew distribution, our method achieves a notable improvement in final accuracy while maintaining a lower standard deviation compared to the baselines.
Under the Dirichlet distribution, where the heterogeneity is relatively moderate, our method still achieves competitive or superior final accuracy, with a lower standard deviation than the compared methods.
Fig.~\ref{fig:pillar}(c) illustrates the accuracy on the UCI-HAR dataset, where data heterogeneity across clients is not as severe as under the Dirichlet and Label Skew distributions. In this setting, all methods achieve final accuracy above \(90\%\), and our method attains comparable accuracy to the baselines while consistently exhibiting lower standard deviation.
Our method improves accuracy while enhancing training stability. The accuracy improvement is more pronounced in challenging scenarios where the baseline accuracy is relatively low, whereas the stability improvement, reflected by reduced standard deviation, is more evident when the accuracy is already high across all methods.

\subsection{Communication and Computation Efficiency}
\label{sec:efficiency}
\begin{figure*}[!t]
    \centering
    \includegraphics[width=0.7\textwidth]{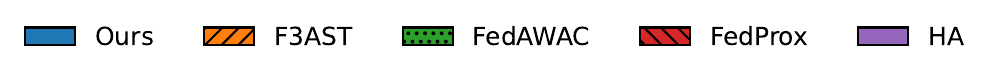}

    \vspace{-0.5em}

    \subfloat[CIFAR-10]{%
        \includegraphics[width=0.23\textwidth]{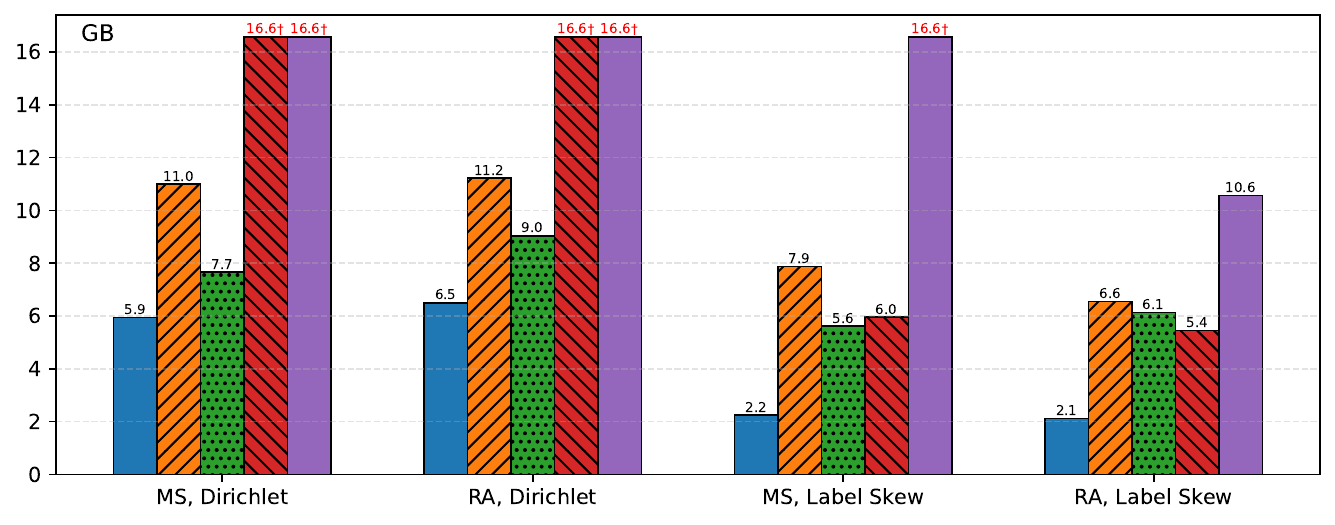}%
        \label{subfig:eff_uplink_cifar}%
    }
    \hfill
    \subfloat[Fashion-MNIST]{%
        \includegraphics[width=0.23\textwidth]{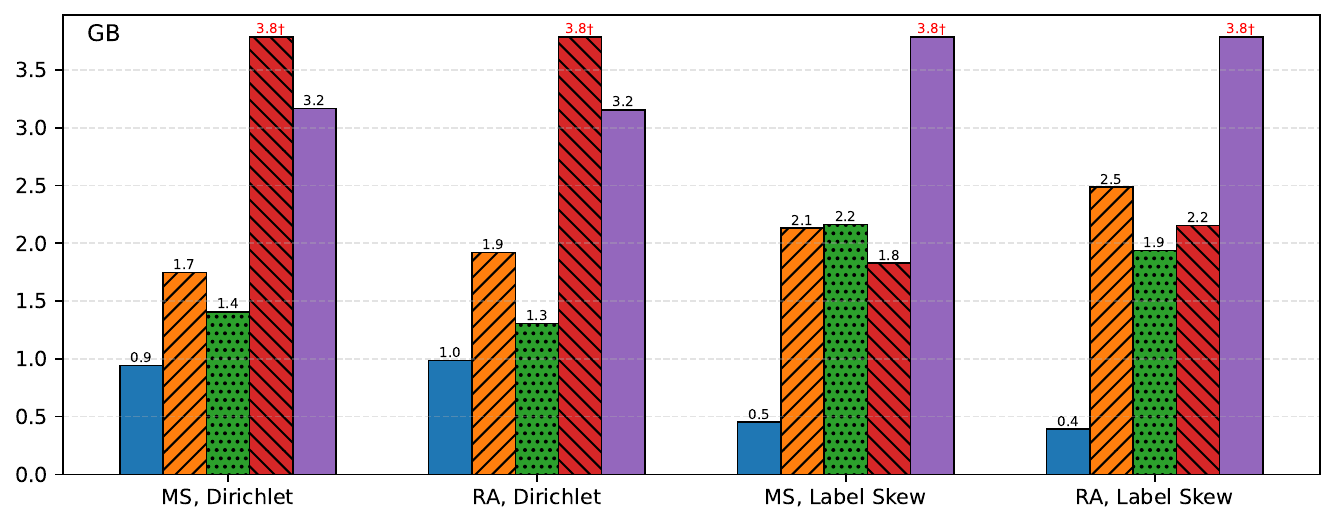}%
        \label{subfig:eff_uplink_fashion}%
    }
    \hfill
    \subfloat[CIFAR-10]{%
        \includegraphics[width=0.23\textwidth]{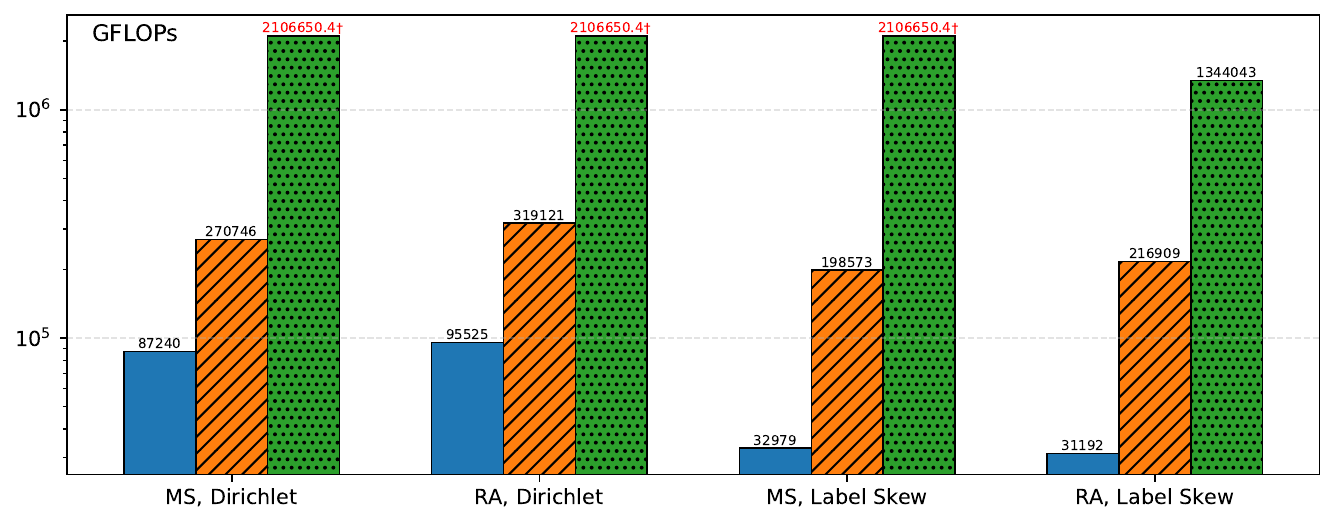}%
        \label{subfig:eff_server_cifar}%
    }
    \hfill
    \subfloat[Fashion-MNIST]{%
        \includegraphics[width=0.23\textwidth]{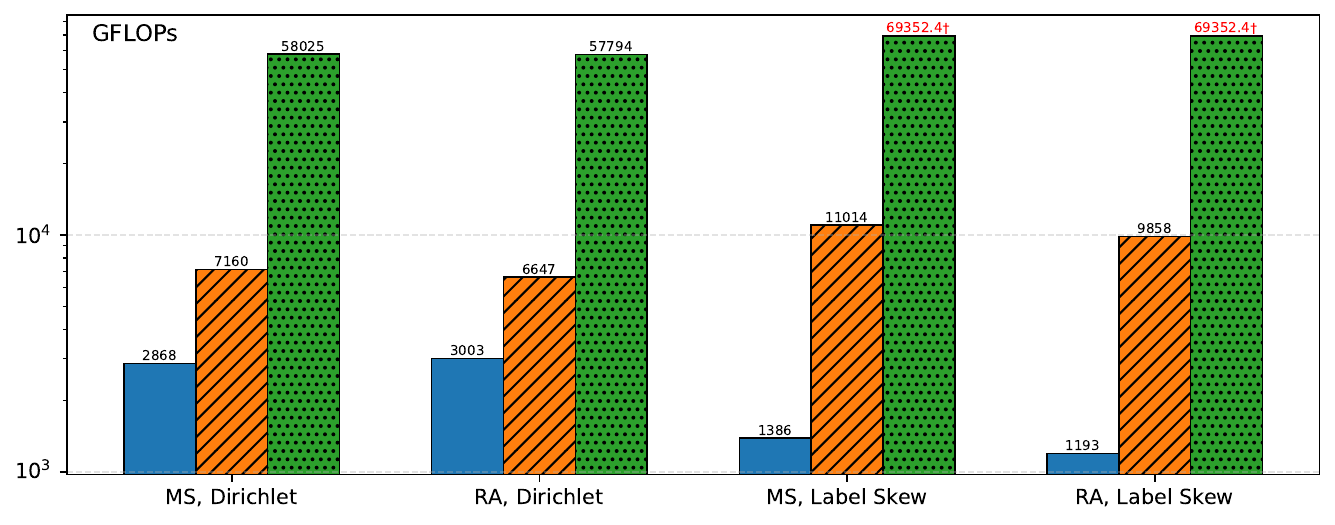}%
        \label{subfig:eff_server_fashion}%
    }

    \caption{Efficiency analysis. (a-b) Cumulative uplink communication cost (GB). (c-d) Cumulative server-side computation cost (GFLOPs, log scale). Lower is better.\label{fig:efficiency}}
\end{figure*}

Beyond accuracy, we evaluate the system-level efficiency of the proposed method using hardware-independent metrics, which decouple the analysis from specific device configurations and are consistent with our assumption of identical client hardware.
As communication is widely recognized as the primary bottleneck in FL systems, we measure the cumulative uplink data volume transmitted per client, which reflects the communication cost for each method.
In addition, we report the server-side floating-point operations (FLOPs) to quantify the additional computational overhead introduced by the DQL-based client selection policy.
Both metrics are measured at the communication round where a method first reaches a target accuracy \(Acc_t\). Fig.~\ref{fig:efficiency} reports these two metrics for all methods across the evaluated settings.

As shown in Fig.~\ref{subfig:eff_uplink_cifar} and~\ref{subfig:eff_uplink_fashion}, our method reduces uplink communication cost compared to baselines when
reaching the same target accuracy \(Acc_t\). Under the Dirichlet and Label Skew settings, the average savings are $27\%$
and $69\%$, respectively. This improvement stems from the faster convergence of our method, consistent with the trends
observed in Fig.~\ref{fig:baseline_curve}.
Fig.~\ref{subfig:eff_server_cifar} and~\ref{subfig:eff_server_fashion} present the server-side computation cost of our method compared to FedAWAC and HA-EdgeFLow, while F3AST and FedProx involve no server-side model forward computation.
Despite the additional computation and training overhead on the server, our method reduces server-side computation by $79\%$ and $92\%$ on average under the Dirichlet and Label Skew settings, respectively.

The above experiments demonstrate that although our proposed method introduces additional computation and communication overhead for Q-network training, the improved convergence enables savings in communication bandwidth, and the extra server computation cost remains reasonable.

\subsection{Impact of Context Length \(H\)}
\begin{figure}[!t]
    \centering
    \includegraphics[width=0.85\columnwidth]{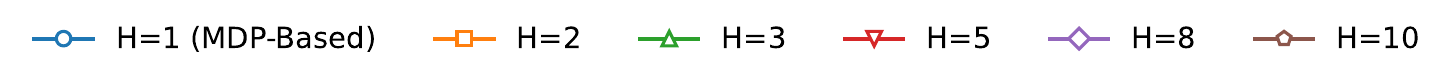}
    \vspace{-0.5em} 
    \subfloat[Dirichlet]{%
        \includegraphics[width=0.45\columnwidth]{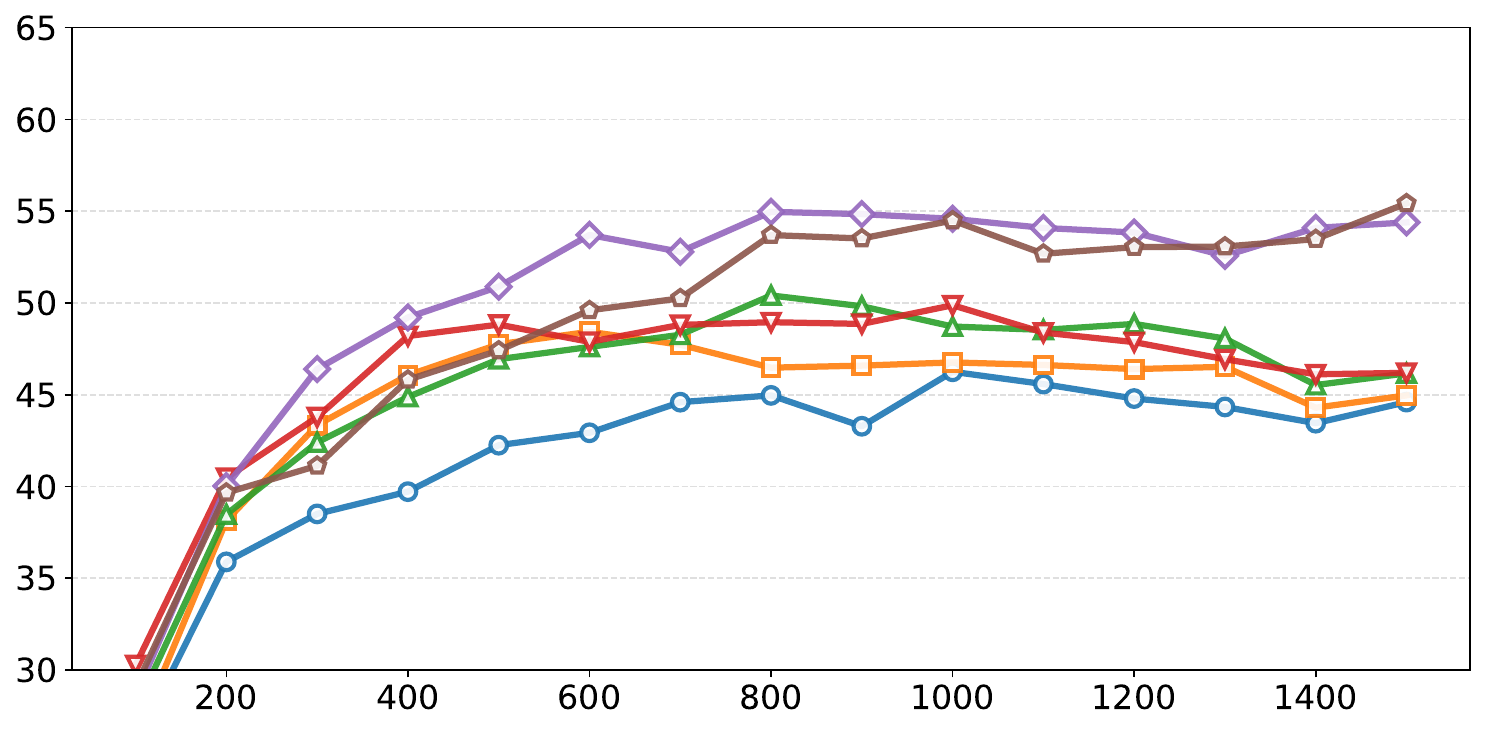}%
        \label{subfig:recalldiri}%
    }
    \hfill
    \subfloat[Label Skew]{%
        \includegraphics[width=0.45\columnwidth]{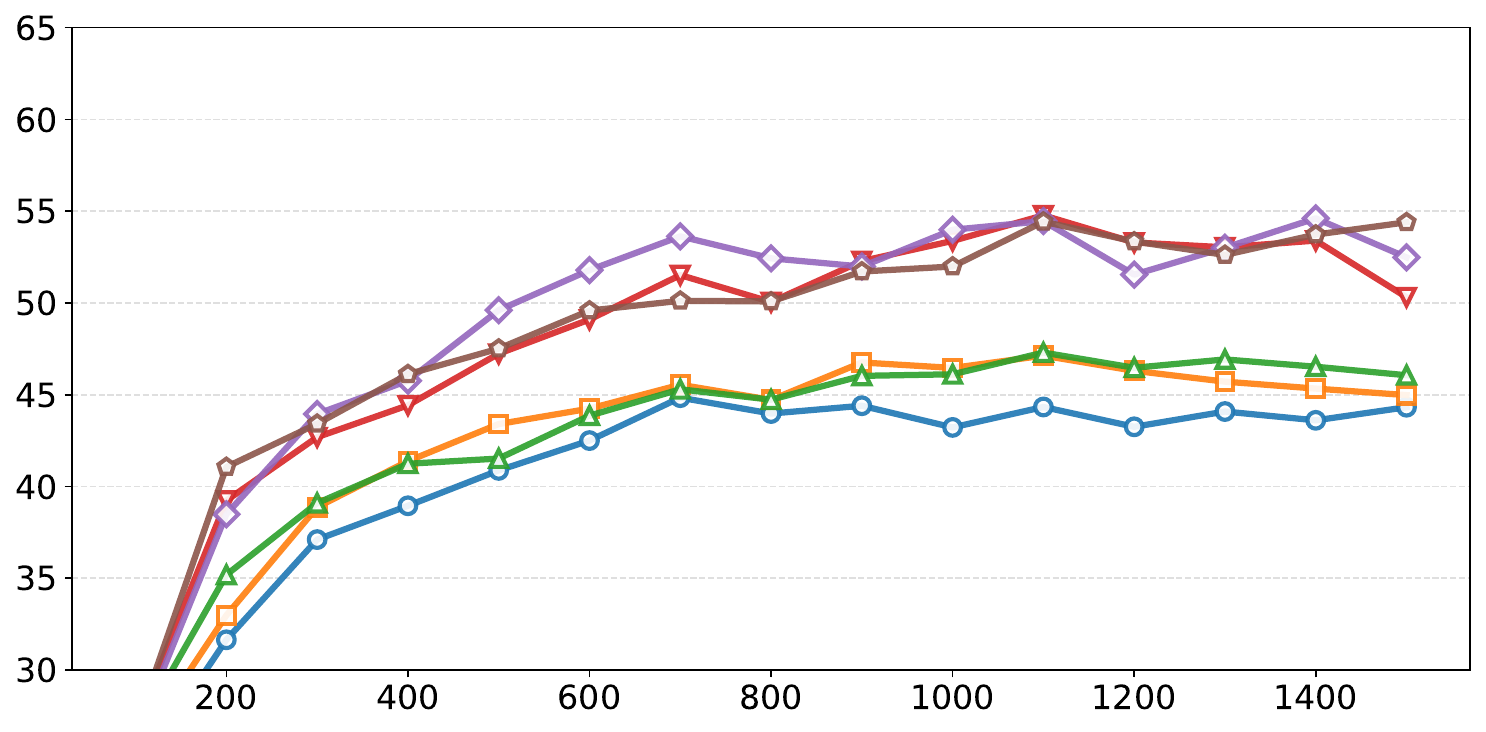}%
        \label{subfig:recallniid}%
    }

    \caption{Accuracy versus communication rounds for different \(H\).\label{fig:recall}}
\end{figure}
In this subsection, we investigate the impact of the length \(H\) of temporal context, defined in Eq.~\eqref{eq:truncated_history}. Since we formulate the client selection problem as a POMDP and introduce temporal context as a solution, the impact of \(H\) validates the necessity of our formulation.
Notably, the POMDP-based solution degenerates to an MDP-based solution when \(H=1\), since the decision no longer relies on the context.
Fig.~\ref{fig:recall} demonstrates the accuracy versus communication rounds of various \(H\) on the CIFAR-10 dataset under the Mobile Server (MS) setting and Label Skew distribution. Similar trends are observed under other settings and thus omitted for brevity.

As illustrated in Fig.~\ref{fig:recall}, the MDP-based case (\(H=1\)) yields the lowest performance, confirming that the MDP-based solution is less suitable for the partial visibility scenario compared to POMDP. The performance improves as \(H\) increases, indicating that incorporating additional historical information enhances the client selection process.
However, further enlarging the context length to \(H > 8\) brings no noticeable performance improvement. This suggests that recent states and actions of the FL server already contain sufficient historical information for effective client selection.

This experiment validates our approach of introducing historical information to address the POMDP. Moreover, we confirm that a limited-length temporal window suffices in practice.

\subsection{Impact of $|\mathcal{C}^t|$, $p$, and $|\mathcal{S}^t|$}
To validate our proposed method under various visibility situations, we investigate the impact of the number of visible clients per communication round (controlled by $|\mathcal{C}^t|$ in MS and $p$ in RA) and the selection size $|\mathcal{S}^t|$.

\begin{figure}[!t]
    \centering

    \subfloat[Dirichlet]{%
        \includegraphics[width=0.45\columnwidth]{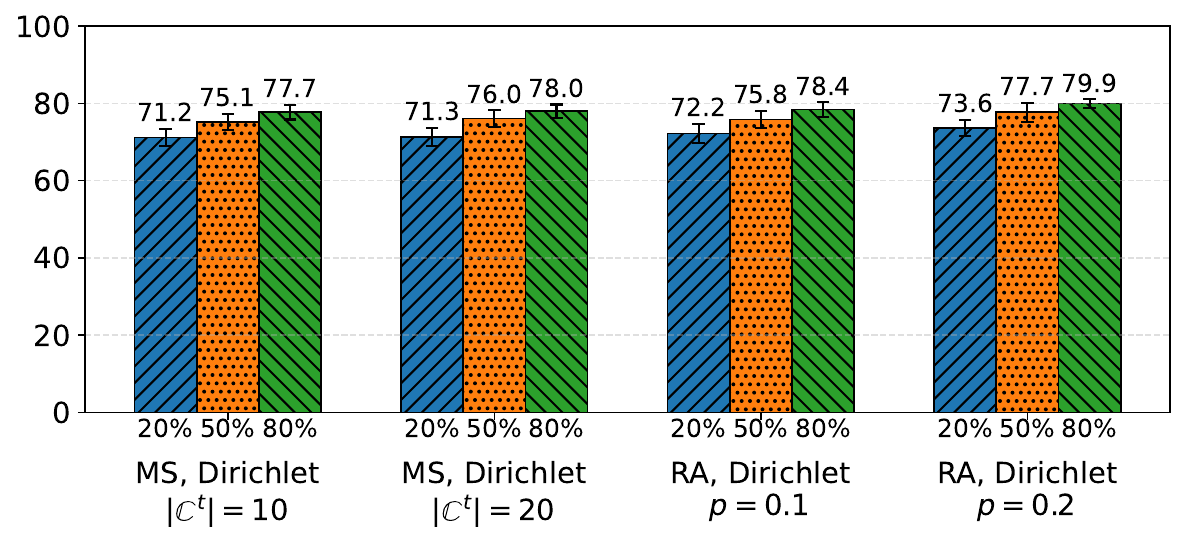}%
        \label{subfig:selbardiri}%
    }
    \hfill 
    \subfloat[Label Skew]{%
        \includegraphics[width=0.45\columnwidth]{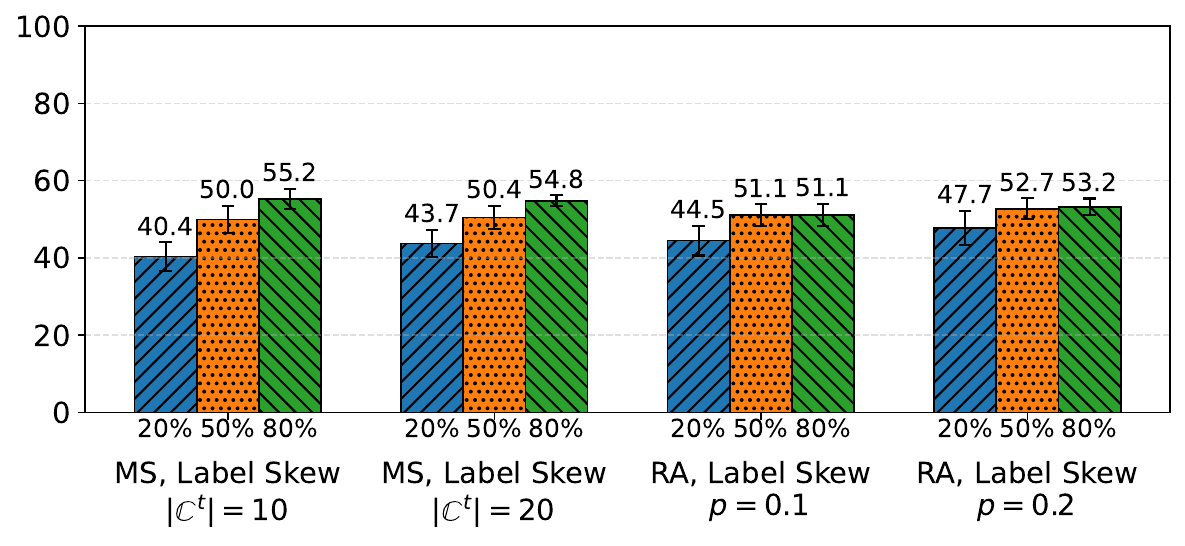}%
        \label{subfig:selbarniid}%
    }

    \caption{Accuracy under varying client selection ratios $K/|\mathcal{C}^t|$. Results are averaged over the final 10 training rounds, where the error bars denote the standard deviation.\label{fig:selbar}}
\end{figure}

As illustrated in Fig.~\ref{fig:selbar}, the performance trends remain consistent across all combinations of visibility and data heterogeneity settings.
First, enlarging either the average number of clients (\ie increasing $|\mathcal{C}^t|$ or $p$) or the selection ratio $|\mathcal{S}^t|/|\mathcal{C}^t|$ improves the test accuracy.
The extent of these performance gains exhibits a uniform pattern across diverse scenarios, and our method performs consistently across different visibility conditions and selection sizes.

\subsection{Ablation Study: Identity-Aware Embedding}
\begin{figure}[!t]
    \centering
    \begin{minipage}{0.64\columnwidth}
        \centering
        \includegraphics[width=\columnwidth]{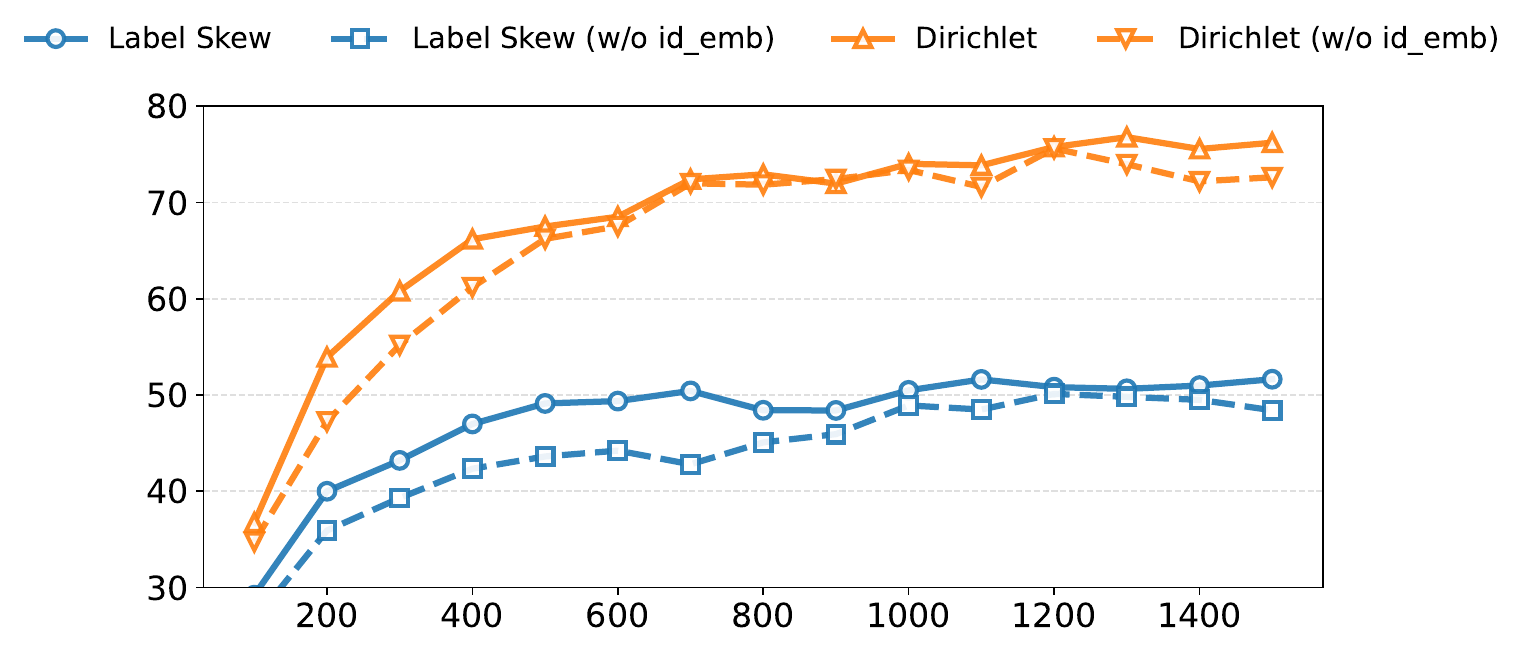}
    \end{minipage}
    \hfill
    \begin{minipage}{0.32\columnwidth}
        \centering
        \includegraphics[width=\columnwidth]{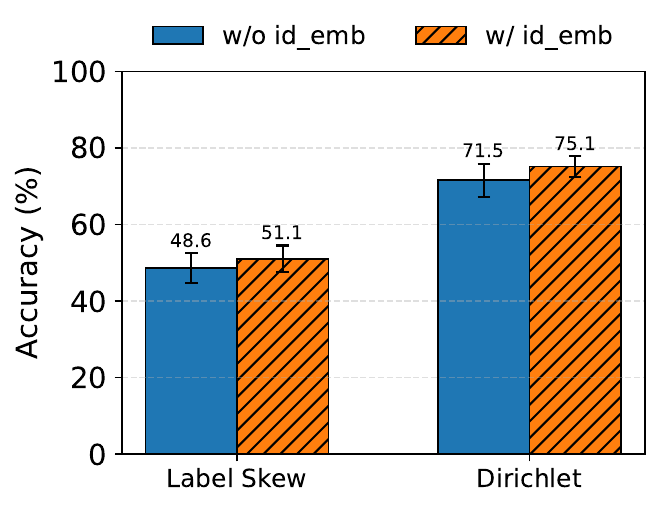}
    \end{minipage}
    \caption{Ablation on CIFAR-10 under Mobile Server setting. Accuracy versus Communication rounds (left) and final accuracy comparison (right). \textit{w/o emb} denotes that the Identity-Aware Embedding module is removed.\label{fig:ablation}}
\end{figure}
In this subsection, we conduct an ablation experiment to validate the effectiveness of the Identity-Aware Embedding module in the Q-network.
As shown in Fig.~\ref{fig:ablation}, removing the Identity-Aware Embedding module leads to a consistent reduction in accuracy throughout the entire training process across all evaluated settings.
Under the Dirichlet distribution, this degradation is further accompanied by a substantial increase in the standard deviation of the final accuracy, indicating reduced training stability.
These results confirm the necessity of the Identity-Aware Embedding module, which contributes to both improved accuracy and enhanced training stability by enabling the Q-network to retain persistent, client-specific characteristics across communication rounds.

\section{Conclusion}
This work established a POMDP framework for federated client selection under partial visibility. We formulated the POMDP elements and developed a multi-step DQL solution, integrating a Spatio-Temporal Attention-based Q-Network to encode historical global models and dynamically handle variable-length inputs.
Experiments across visibility and heterogeneity settings show that our method improves accuracy and training stability in most scenarios and reduces computational and communication cost.
Ablation studies on context length and Identity-Aware embeddings further validated the Spatio-Temporal Attention architecture's design choices.
By adapting training to the time-varying set of clients visible to the server, the framework is well suited to IoT and edge deployments where device availability and server coverage fluctuate over time.

\textbf{Limitations and Future Work.} While effective across the studied settings, testing the framework on larger and more diverse deployments, as well as on broader task domains, is the next step. The convergence guarantee also rests on standard simplifying assumptions, leaving room for stronger theoretical results. We additionally expect that more expressive modeling of client interactions and temporal context could yield further gains, which we leave to future work.

\bibliographystyle{IEEEtran}
\bibliography{refs}

\appendix

\subsection{Proof of Theorem~\ref{th:conv_sel_flow}} \label{app:conv_sel_flow}
To ensure the clarity of the proof, we first establish the necessary lemmas in Appendix~\ref{appa:lemma_flow}, before proceeding to the proof of the main result in Appendix~\ref{appa:thproof_flow}. The proofs of the lemmas are deferred to Appendix~\ref{appa:lemmaproof_flow}.

\subsection{Key Lemmas} \label{appa:lemma_flow}

\begin{lemma} \label{lem:1_flow}
    Under Assumption~\ref{ass:lsmooth_flow}(i), it follows that
    \begin{equation} \label{eq:lemma:1_flow}
        \begin{split}
            \E{\inner{\nabla F(W_{glob}^t)}{\nabla F_{\mathcal{C}^t}(W_{glob}^t) - \tilde{g}_{n,j}^t}} \qquad \qquad \qquad \\
            \leq \frac{1}{4}\E{\normsq{\nabla F(W_{glob}^t)}} + \frac{L^2}{|\mathcal{C}^t|}\sum_{n \in \mathcal{C}^t} \normsq{W_{glob}^t - W_{n,j}^t}
        \end{split}
    \end{equation}
\end{lemma}

\begin{lemma} \label{lem:2_flow}
    Under Assumption~\ref{ass:niid_flow}, for the selected subset $\mathcal{S}^t \subset \mathcal{C}^t$, it follows that
    \begin{equation} \label{eq:lemma:2_flow}
        \begin{split}
            \mathbb{E}&\Big[\inner{\nabla F(W_{glob}^t)}{\nabla F(W_{glob}^t) - \nabla F_{\mathcal{S}^t}(W_{glob}^t)}\Big] \\
            \leq{} &\frac{1}{2}\E{\normsq{\nabla F(W_{glob}^t)}} - \frac{1}{2}\E{\normsq{\nabla F_{\mathcal{S}^t}(W_{glob}^t)}} \\
            &+ d_{\mathcal{C}^t}^2 + \normsq{\nabla F_{\mathcal{C}^t}(W_{glob}^t) - \nabla F_{\mathcal{S}^t}(W_{glob}^t)}
        \end{split}
    \end{equation}
    where $\nabla F_{\mathcal{S}^t}(\cdot) \triangleq \frac{1}{|\mathcal{S}^t|}\sum_{n \in \mathcal{S}^t} \nabla F_n(\cdot)$.
\end{lemma}

\begin{lemma} \label{lem:3_flow}
    Under Assumption~\ref{ass:gradient_flow}, it follows that
    \begin{equation} \label{eq:lemma:3_flow}
        \frac{1}{|\mathcal{C}^t|}\sum_{n \in \mathcal{C}^t} \normsq{W_{glob}^t - W_{n,j}^t} \leq j^2\eta^2 G^2
    \end{equation}
\end{lemma}

\begin{lemma} \label{lem:4_flow}
    Under Assumptions~\ref{ass:lsmooth_flow}--\ref{ass:qstrong}, the selection-induced gradient deviation from Q-value-based client selection at rate $c$ satisfies (viewing the Q-network as a function of the gradient, \ie $q_n = Q(W_n) \triangleq Q(\nabla F_n(W_{glob}^t))$, hereafter written $Q(\nabla F_n)$ directly):
    \begin{equation}
        \begin{split}
            \normsq{\nabla F_{\mathcal{C}^t}(W_{glob}^t) - \nabla F_{\mathcal{S}^t}(W_{glob}^t)}
            &\leq \frac{2L_Q}{\mu c}\Gamma_{\mathcal{C}^t}^2 \\
            &+ \frac{4(L_Q + \mu c)}{\mu^2 c}\Delta_Q^t
        \end{split}
        \label{eq:eps_bound}
    \end{equation}
    where $\Gamma_{\mathcal{C}^t}^2 \triangleq \mathbb{E}_{n \in \mathcal{C}^t}[\normsq{\nabla F_n - \nabla F_{\mathcal{C}^t}}]$ is the intra-cluster gradient variance and $\Delta_Q^t \triangleq Q^* - Q(\nabla F_{\mathcal{C}^t}) \geq 0$ is the Q-value alignment gap.
\end{lemma}

\subsection{Proof of the Main Result} \label{appa:thproof_flow}
\begin{proof}
The overall model update after one training round with selection ratio $c = |\mathcal{S}^t| / |\mathcal{C}^t|$ is
\begin{equation}
    W_{glob}^{t+1} = W_{glob}^t - \frac{\eta}{|\mathcal{S}^t|}\sum_{n \in \mathcal{S}^t}\sum_{j=0}^{k-1}\tilde{g}_{n,j}^t
\end{equation}
By Assumption~\ref{ass:lsmooth_flow}(i), we have
\begin{align}
    & \mathbb{E}\big[F(W_{glob}^{t+1}) - F(W_{glob}^{t})\big] \nonumber \\
    & \leq \E{\inner{\nabla F(W_{glob}^t)}{W_{glob}^{t+1} - W_{glob}^{t}}} \nonumber \\
    & + \frac{L}{2}\E{\normsq{W_{glob}^{t+1} - W_{glob}^{t}}} \\
    \begin{split}
        & = -\frac{\eta}{c|\mathcal{C}^t|} \sum_{n \in \mathcal{S}^t} \sum_{j=0}^{k-1} \E{\inner{\nabla F}{\tilde{g}_{n,j}^t}} \\
        & + \frac{L\eta^2}{2(c|\mathcal{C}^t|)^2} \E{\normsq{\sum_{n \in \mathcal{S}^t} \sum_{j=0}^{k-1} \tilde{g}_{n,j}^t}}
    \end{split}
    \label{eq:lsmooth_sel}
\end{align}

First, we bound the first term in Eq.~\eqref{eq:lsmooth_sel}. Due to the fact that
\begin{equation}
    \begin{split}
        \inner{\nabla F}{\tilde{g}_{n,j}^t} &= \inner{\nabla F}{\tilde{g}_{n,j}^t - \nabla F_{\mathcal{S}^t}} \\
        &+ \inner{\nabla F}{\nabla F_{\mathcal{S}^t} - \nabla F} + \inner{\nabla F}{\nabla F}
    \end{split}
\end{equation}
and based on Lemma~\ref{lem:1_flow} and Lemma~\ref{lem:2_flow}, substituting Eq.~\eqref{eq:lemma:1_flow} and Eq.~\eqref{eq:lemma:2_flow} yields
\begin{align}
    -&\frac{\eta}{c|\mathcal{C}^t|} \sum_{n \in \mathcal{S}^t} \sum_{j} \E{\inner{\nabla F}{\tilde{g}_{n,j}^t}} \nonumber \\
    \begin{split}
        \leq{} &-\frac{k\eta}{4}\E{\normsq{\nabla F}} + k\eta\,d_{\mathcal{C}^t}^2 + k\eta\,\normsq{\nabla F_{\mathcal{C}^t} - \nabla F_{\mathcal{S}^t}} \\
        &+ \frac{L^2 \eta}{(c|\mathcal{C}^t|)^2} \sum_{n \in \mathcal{S}^t}\sum_{j} \normsq{W_{glob}^t - W_{n,j}^t} \\
        &- \frac{\eta}{2} \sum_{j} \E{\normsq{\nabla F_{\mathcal{S}^t}}}
    \end{split}
    \label{eq:t2_term1_flow}
\end{align}

Next, we bound the second term in Eq.~\eqref{eq:lsmooth_sel}. We have
\begin{align}
    & \frac{L\eta^2}{2(c|\mathcal{C}^t|)^2}\E{\normsq{\sum_{n \in \mathcal{S}^t}\sum_{j}\tilde{g}_{n,j}^t}} \nonumber \\
    \begin{split}
        \overset{\text{\ding{192}}}{=}{} &\frac{L\eta^2}{2(c|\mathcal{C}^t|)^2}\E{\normsq{\sum_{n}\sum_{j}\left(\tilde{g}_{n,j}^t - \E{\tilde{g}_{n,j}^t}\right)}} \\
        &+ \frac{L\eta^2}{2(c|\mathcal{C}^t|)^2}\normsq{\sum_{n}\sum_{j}\E{\tilde{g}_{n,j}^t}}
    \end{split}
    \\
    \begin{split}
        \overset{\text{\ding{193}}}{\leq}{} &\frac{Lk\eta^2\sigma^2}{2c|\mathcal{C}^t|} \\
        &+ \frac{Lk\eta^2}{2}\sum_{j}\E{\normsq{\nabla F_{\mathcal{S}^t}}}
    \end{split}
    \label{eq:t2_term2_flow}
\end{align}
where \ding{192} holds due to the fact that $\mathbb{E}\normsq{X} = \E{\normsq{X - \mathbb{E}X}} + \normsq{\mathbb{E}X}$, and \ding{193} follows from Fact~\ref{fact:1_flow}, Assumption~\ref{ass:gradient_flow}, and Jensen's inequality.

With upper bounds established for both terms in Eq.~\eqref{eq:lsmooth_sel}, Eq.~\eqref{eq:t2_term1_flow} and Eq.~\eqref{eq:t2_term2_flow} can be substituted back into Eq.~\eqref{eq:lsmooth_sel}, resulting in
\begin{align}
    \mathbb{E}\big[ & F(W_{glob}^{t+1}) - F(W_{glob}^{t})\big] \leq -\frac{k\eta}{4}\E{\normsq{\nabla F}} \nonumber \\
    &+ k\eta\,d_{\mathcal{C}^t}^2 + k\eta\,\normsq{\nabla F_{\mathcal{C}^t} - \nabla F_{\mathcal{S}^t}} \nonumber \\
    &+ \frac{Lk\eta^2\sigma^2}{2c|\mathcal{C}^t|} + \frac{\eta}{2}\left(Lk\eta - 1\right)\sum_{j}\E{\normsq{\nabla F_{\mathcal{S}^t}}} \label{eq:th_combined_flow}
\end{align}
Since $Lk\eta < 1$, the last term in Eq.~\eqref{eq:th_combined_flow} is negative and can be omitted. By Lemma~\ref{lem:3_flow}, substituting Eq.~\eqref{eq:lemma:3_flow} and using $\sum_{j=0}^{k-1}j^2 \leq \frac{k^3}{3}$, we obtain
\begin{equation}
    \begin{split}
        \mathbb{E}\big[F(W_{glob}^{t+1})& - F(W_{glob}^{t})\big] \leq -\frac{k\eta}{4}\E{\normsq{\nabla F}} \\
        &+ k\eta\,d_{\mathcal{C}^t}^2 + k\eta\,\normsq{\nabla F_{\mathcal{C}^t} - \nabla F_{\mathcal{S}^t}} \\
        &+ \frac{Lk\eta^2\sigma^2}{2c|\mathcal{C}^t|} + \frac{L^2k^3\eta^3G^2}{3}
    \end{split}
    \label{eq:th_perround_flow}
\end{equation}

Consider the time average $\frac{1}{T}\sum_{t=0}^{T-1}[\cdot]$ on both sides of Eq.~\eqref{eq:th_perround_flow}. Denoting the optimum of $F(\cdot)$ as $F^*$ and dividing by $\frac{k\eta}{4}$, we have
\begin{equation}
    \begin{split}
        \frac{1}{T}\sum_{t=0}^{T-1}\E{\normsq{\nabla F(W_{glob}^t)}} &\leq{} \frac{4}{k\eta T}\left(F(W_{glob}^0) - F^*\right) \\
        &+ 4\bar{d}^2 + \frac{4}{T}\sum_{t=0}^{T-1}\normsq{\nabla F_{\mathcal{C}^t} - \nabla F_{\mathcal{S}^t}} \\
        &+ \frac{2L\eta\sigma^2}{c\bar{|\mathcal{C}|}} + \frac{4L^2k^2\eta^2G^2}{3}
    \end{split}
    \label{eq:th_before_subst_flow}
\end{equation}
Substituting Eq.~\eqref{eq:eps_bound} into the third term of Eq.~\eqref{eq:th_before_subst_flow} yields Eq.~\eqref{eq:conv_sel_flow}, which completes the proof.
\end{proof}

\subsection{Proof of the Lemmas} \label{appa:lemmaproof_flow}
\label{app:lemmas_flow}
We will use the following facts in proving the results.

\begin{fact} \label{fact:1_flow}
    Let $\{\xi_i\}_{i=1}^{n}$ be a sequence of random variables, and the vector sequence $\{x_i\}_{i=1}^{n}$ satisfy that each $x_i \in \mathbb{R}^d$ is a function of $\{\xi_i\}_{i=1}^{n}$. Suppose that the conditional expectation of $x_i$ is $\E{x_i\mid\xi_{i-1}, \cdots, \xi_{1}} = e_i$, then we have
    \begin{equation}
        \E{\normsq{\sum_{i=1}^{n}(x_i - e_i)}} = \sum_{i=1}^{n} \E{\normsq{x_i - e_i}}
    \end{equation}
\end{fact}

\begin{proof}
    According to the fact that for any vector $a_i$,
    \begin{equation}
        \normsq{\sum_{i=1}^{n} a_i} = \sum_{i=1}^{n}\normsq{a_i} + 2\sum_{1 \leq i < j \leq n}\inner{a_i}{a_j}
    \end{equation}
    Assuming that $a_i = x_i - e_i$, we obtain
    \begin{equation}
        \begin{split}
            \E{\normsq{\sum_{i=1}^{n}(x_i - e_i)}} = \sum_{i=1}^{n} \E{\normsq{x_i - e_i}} \qquad \\
            + 2\sum_{1 \leq i < j \leq n} \E{(x_i - e_i)^\top (x_j - e_j)}
        \end{split}
    \end{equation}
    The law of total expectation implies that
    \begin{align}
        &\E{(x_i - e_i)^\top (x_j - e_j)} \nonumber \\
        & = \E{(x_i - e_i)^\top\E{(x_j - e_j)\mid\xi_{j-1}, \cdots, \xi_{1}}} \overset{\text{\ding{192}}}{=} 0
    \end{align}
    where \ding{192} follows from the conditional zero-mean property of $x_j - e_j$. Therefore, the cross terms vanish, which completes the proof.
\end{proof}

\begin{fact} \label{fact:2_flow}
    For any vectors $a, b \in \mathbb{R}^d$ and scalar $x > 0$, the following inequality holds:
    \begin{equation}
        \inner{a}{b} \leq \frac{x}{2}\normsq{a} + \frac{1}{2x}\normsq{b}
    \end{equation}
\end{fact}

\begin{proof}
    Cauchy-Schwarz inequality gives that $\inner{a}{b} \leq \|a\|\|b\|$, and Young's inequality gives that $uv \leq \frac{u^p}{p} + \frac{v^q}{q}$ for any $u,v \geq 0$ and $p,q > 1$ such that $\frac{1}{p} + \frac{1}{q} = 1$. Taking $p = q = 2$, $u = \sqrt{x}\|a\|$, and $v = \frac{\|b\|}{\sqrt{x}}$ yields $\|a\|\|b\| \leq \frac{x}{2}\normsq{a} + \frac{1}{2x}\normsq{b}$, which completes the proof.
\end{proof}

\noindent
\textbf{Proof of Lemma~\ref{lem:1_flow}.}
\begin{proof}
    According to Fact~\ref{fact:2_flow} and assume $x = \frac{1}{2}$, we have $\inner{a}{b} \leq \frac{1}{4}\normsq{a} + \normsq{b}$. With $a = \nabla F$ and $b = \nabla F_{\mathcal{C}^t} - \tilde{g}_{n,j}^t$, we obtain
    \begin{equation} \label{eq:l1pf1_flow}
        \begin{split}
            &\E{\inner{\nabla F}{\nabla F_{\mathcal{C}^t} - \tilde{g}_{n,j}^t}} \\
            &\leq \frac{1}{4}\E{\normsq{\nabla F}} + \E{\normsq{\nabla F_{\mathcal{C}^t} - \tilde{g}_{n,j}^t}}
        \end{split}
    \end{equation}
    For the second term, we have
    \begin{align}
        &\E{\normsq{\nabla F_{\mathcal{C}^t} - \tilde{g}_{n,j}^t}} \nonumber \\
        &\quad = \normsq{\frac{1}{|\mathcal{C}^t|}\sum_{n \in \mathcal{C}^t} \left(\nabla F_n(W_{glob}^t) - \mathbb{E}_\xi[\tilde{g}_{n,j}^t]\right)} \\
        &\quad \overset{\text{\ding{192}}}{\leq} \frac{L^2}{|\mathcal{C}^t|}\sum_{n \in \mathcal{C}^t} \normsq{W_{glob}^t - W_{n,j}^t} \label{eq:l1pf2_flow}
    \end{align}
    where \ding{192} follows from the identity $\mathbb{E}_\xi[\tilde{g}_{n,j}^t] = \nabla F_n(W_{n,j}^t)$, Jensen's inequality, and Assumption~\ref{ass:lsmooth_flow}(i). Substitute Eq.~\eqref{eq:l1pf2_flow} into Eq.~\eqref{eq:l1pf1_flow} to complete the proof.
\end{proof}

\noindent
\textbf{Proof of Lemma~\ref{lem:2_flow}.}
\begin{proof}
    According to the fact that $2\inner{a}{b} = \normsq{a} + \normsq{b} - \normsq{a-b}$ for any vectors $a$ and $b$. With $a = \nabla F$ and $b = \nabla F - \nabla F_{\mathcal{S}^t}$, we obtain
    \begin{align}
        \mathbb{E}&\Big[\inner{\nabla F}{\nabla F - \nabla F_{\mathcal{S}^t}}\Big] \nonumber \\
        &= \frac{1}{2}\E{\normsq{\nabla F}} - \frac{1}{2}\E{\normsq{\nabla F_{\mathcal{S}^t}}} + \frac{1}{2}\E{\normsq{\nabla F - \nabla F_{\mathcal{S}^t}}} \nonumber \\
        &\overset{\text{\ding{192}}}{\leq} \frac{1}{2}\E{\normsq{\nabla F}} - \frac{1}{2}\E{\normsq{\nabla F_{\mathcal{S}^t}}} \nonumber \\
        & \quad + \normsq{\nabla F - \nabla F_{\mathcal{C}^t}} + \normsq{\nabla F_{\mathcal{C}^t} - \nabla F_{\mathcal{S}^t}} \nonumber \\
        &\overset{\text{\ding{193}}}{\leq} \frac{1}{2}\E{\normsq{\nabla F}} - \frac{1}{2}\E{\normsq{\nabla F_{\mathcal{S}^t}}} \nonumber \\
        & \quad + d_{\mathcal{C}^t}^2 + \normsq{\nabla F_{\mathcal{C}^t} - \nabla F_{\mathcal{S}^t}}
    \end{align}
    where \ding{192} follows from the triangle inequality $\normsq{\nabla F - \nabla F_{\mathcal{S}^t}} \leq 2\normsq{\nabla F - \nabla F_{\mathcal{C}^t}} + 2\normsq{\nabla F_{\mathcal{C}^t} - \nabla F_{\mathcal{S}^t}}$, and \ding{193} follows from Assumption~\ref{ass:niid_flow}.
\end{proof}

\noindent
\textbf{Proof of Lemma~\ref{lem:3_flow}.}
\begin{proof}
    From the local update rule, $W_{glob}^t - W_{n,j}^t = \eta\sum_{i=0}^{j-1}\tilde{g}_{n,i}^t$. Therefore,
    \begin{align}
        & \frac{1}{|\mathcal{C}^t|}\sum_{n \in \mathcal{C}^t}\normsq{W_{glob}^t - W_{n,j}^t} = \eta^2\E{\normsq{\sum_{i=0}^{j-1}\tilde{g}_{n,i}^t}} \\
        &\quad = j^2\eta^2\E{\normsq{\frac{1}{j}\sum_{i=0}^{j-1}\tilde{g}_{n,i}^t}} \\
        &\quad \overset{\text{\ding{192}}}{\leq} j\eta^2\sum_{i=0}^{j-1}\E{\normsq{\tilde{g}_{n,i}^t}} \overset{\text{\ding{193}}}{\leq} j^2\eta^2G^2
    \end{align}
    where \ding{192} follows from Jensen's inequality, and \ding{193} follows from Assumption~\ref{ass:gradient_flow}.
\end{proof}

\noindent
\textbf{Proof of Lemma~\ref{lem:4_flow}.}
\begin{proof}
    Define the gradient deviation $h_n \triangleq \nabla F_n(W_{glob}^t) - \nabla F_{\mathcal{C}^t}(W_{glob}^t)$, noting that $\frac{1}{|\mathcal{C}^t|}\sum_{n \in \mathcal{C}^t} h_n = 0$. Substituting $\nabla F_n = \nabla F_{\mathcal{C}^t} + h_n$, the selection-induced gradient deviation is
    \begin{equation}
        \nabla F_{\mathcal{C}^t} - \nabla F_{\mathcal{S}^t} = -\frac{1}{|\mathcal{S}^t|}\sum_{n \in \mathcal{S}^t} h_n
    \end{equation}
    By Cauchy-Schwarz inequality,
    \begin{equation}
        \normsq{\nabla F_{\mathcal{C}^t} - \nabla F_{\mathcal{S}^t}} = \left\|\frac{1}{c|\mathcal{C}^t|}\sum_{n \in \mathcal{S}^t} h_n\right\|^2 \leq \frac{1}{c|\mathcal{C}^t|}\sum_{n \in \mathcal{S}^t} \normsq{h_n}
        \label{eq:eps_cs}
    \end{equation}

    By the triangle inequality we have $\normsq{h_n} \leq 2\normsq{\nabla F_n - v^*} + 2\normsq{v^* - \nabla F_{\mathcal{C}^t}}$. By Assumption~\ref{ass:qstrong} with $x = \nabla F_n$:
    \begin{equation}
        \normsq{\nabla F_n - v^*} \leq \frac{2}{\mu}(Q(v^*) - Q(\nabla F_n))
    \end{equation}
    Similarly with $x = \nabla F_{\mathcal{C}^t}$:
    \begin{equation}
        \normsq{v^* - \nabla F_{\mathcal{C}^t}} \leq \frac{2}{\mu}(Q(v^*) - Q(\nabla F_{\mathcal{C}^t})) = \frac{2\Delta_Q^t}{\mu}
    \end{equation}
    Combining both bounds yields $\normsq{h_n} \leq \frac{4}{\mu}(Q(v^*) - Q(\nabla F_n)) + \frac{4\Delta_Q^t}{\mu}$. Substituting into Eq.~\eqref{eq:eps_cs}:
    \begin{equation}
        \normsq{\nabla F_{\mathcal{C}^t} - \nabla F_{\mathcal{S}^t}} \leq \frac{4}{\mu}(Q(v^*) - \bar{Q}_{\mathcal{S}}) + \frac{4}{\mu}\Delta_Q^t
        \label{eq:eps_qterms}
    \end{equation}
    where $\bar{Q}_{\mathcal{S}} = \frac{1}{|\mathcal{S}^t|}\sum_{n \in \mathcal{S}^t} Q(\nabla F_n)$.

    By the top-$c$ selection rule, $\bar{Q}_{\mathcal{S}} \geq \bar{Q}$ where $\bar{Q} = \frac{1}{|\mathcal{C}^t|}\sum_n Q(\nabla F_n)$. From $\bar{Q} = c\bar{Q}_{\mathcal{S}} + (1-c)\bar{Q}_{\bar{\mathcal{S}}}$ and $\bar{Q}_{\bar{\mathcal{S}}} \leq Q(v^*)$:
    \begin{equation}
        Q(v^*) - \bar{Q}_{\mathcal{S}} \leq \frac{Q(v^*) - \bar{Q}}{c}
        \label{eq:eps_rate}
    \end{equation}

    By Assumption~\ref{ass:lsmooth_flow}(ii) at $x = v^*$ ($\nabla Q(v^*) = 0$), averaging over all $n$:
    \begin{equation}
        \bar{Q} \geq Q(v^*) - \frac{L_Q}{2}\cdot\frac{1}{|\mathcal{C}^t|}\sum_n\normsq{\nabla F_n - v^*}
    \end{equation}
    By the bias-variance decomposition $\frac{1}{|\mathcal{C}^t|}\sum_n\normsq{\nabla F_n - v^*} = \Gamma_{\mathcal{C}^t}^2 + \normsq{v^* - \nabla F_{\mathcal{C}^t}} \leq \Gamma_{\mathcal{C}^t}^2 + \frac{2\Delta_Q^t}{\mu}$:
    \begin{equation}
        Q(v^*) - \bar{Q} \leq \frac{L_Q}{2}\Gamma_{\mathcal{C}^t}^2 + \frac{L_Q}{\mu}\Delta_Q^t
        \label{eq:eps_jensen}
    \end{equation}

    Substituting Eq.~\eqref{eq:eps_rate} and Eq.~\eqref{eq:eps_jensen} into Eq.~\eqref{eq:eps_qterms}:
    \begin{equation}
        \begin{split}
            \normsq{\nabla F_{\mathcal{C}^t} - \nabla F_{\mathcal{S}^t}} &\leq \frac{4}{\mu c}\left(\frac{L_Q}{2}\Gamma_{\mathcal{C}^t}^2 + \frac{L_Q}{\mu}\Delta_Q^t\right) + \frac{4}{\mu}\Delta_Q^t \\
            &= \frac{2L_Q}{\mu c}\Gamma_{\mathcal{C}^t}^2 + \frac{4(L_Q + \mu c)}{\mu^2 c}\Delta_Q^t
        \end{split}
    \end{equation}
    which completes the proof.
\end{proof}

\end{document}